%% file: main.tex
\documentclass[10pt,twocolumn,letterpaper]{article}

\usepackage[pagenumbers]{iccv} 

\input{preamble}
\input{math_commands.tex}

\definecolor{iccvblue}{rgb}{0.21,0.49,0.74}
\usepackage[pagebackref,breaklinks,colorlinks,allcolors=iccvblue]{hyperref}
\usepackage{lipsum}
\usepackage{graphicx}
\usepackage{algorithm}
\usepackage{algpseudocode}
\usepackage{thmtools,thm-restate}
\usepackage{amsthm}
\usepackage{color, colortbl}

\def\paperID{13088} 
\def\confName{ICCV}
\def\confYear{2025}

\def\rmA{{\boldsymbol{A}}}
\def\rmI{{\boldsymbol{I}}}
\def\eqref#1{Eq.~(\ref{#1})}

\title{FlowDPS~:~Flow-Driven Posterior Sampling for Inverse Problems}

\author{
Jeongsol Kim$^{* \, 1}$ \quad Bryan Sangwoo Kim$^{* \, 2}$ \quad Jong Chul Ye$^2$ \\
$^1$Department of Bio and Brain Engineering, KAIST \\
$^2$Kim Jaechul Graduate School of AI, KAIST \\
{\tt\small \{jeongsol, bryanswkim, jong.ye\}@kaist.ac.kr} \\
$^*$ \small Equal contribution}

\begin{document}

\twocolumn[{%
\renewcommand\twocolumn[1][]{#1}%
\maketitle
\vspace{-10mm}
\begin{center}
    \centering
    \captionsetup{type=figure}
    \includegraphics[width=0.93\linewidth]{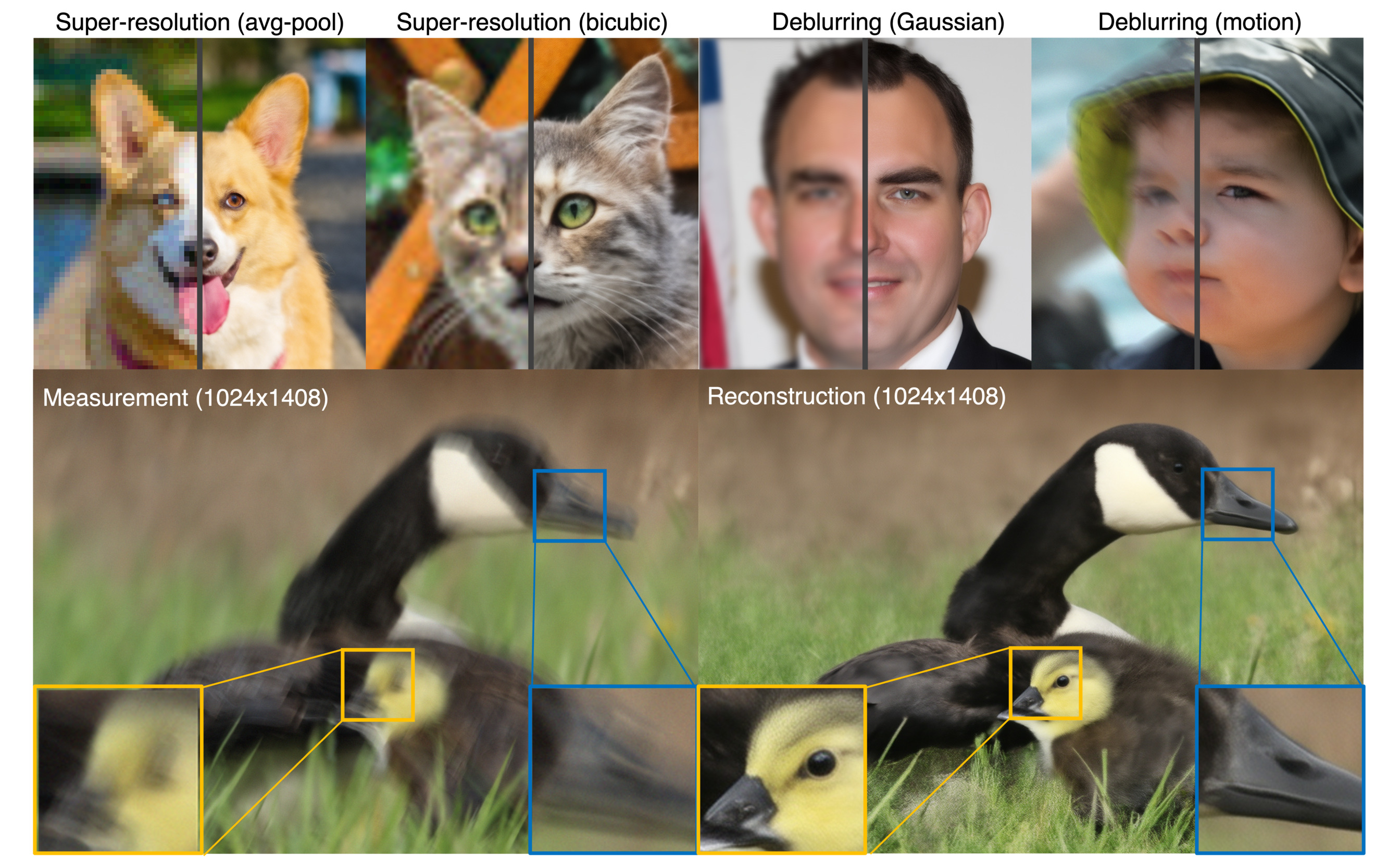}
    \vspace*{-0.3cm}
\caption{{(Top) Super-resolution and deblurring results for linear inverse problems before and after FlowDPS. (Bottom) Results for motion deblurring on a high-resolution image}.}
\label{fig:main}
\end{center}}]

\begin{abstract}
Flow matching is a recent state-of-the-art framework for generative modeling based on ordinary differential equations (ODEs). While closely related to diffusion models, it provides a more general perspective on generative modeling.
Although inverse problem solving has been extensively explored using diffusion models, it has not been rigorously examined within the broader context of flow models. 
Therefore,  here we  extend the diffusion inverse solvers (DIS) --- which perform posterior sampling by combining a denoising diffusion prior with an  likelihood gradient --- into the flow framework. Specifically, by driving the flow-version of Tweedie's formula,
we decompose the flow ODE into two components: one for clean image estimation and the other for noise estimation.
By integrating the likelihood gradient and stochastic noise into each component, respectively, we demonstrate that posterior sampling for inverse problem solving can be effectively achieved using flows. Our proposed solver, Flow-Driven Posterior Sampling (FlowDPS), can also be seamlessly integrated into a latent flow model with a transformer architecture. Across four linear inverse problems, we confirm that FlowDPS outperforms state-of-the-art alternatives, all without requiring additional training.
Code : \url{https://github.com/FlowDPS-Inverse/FlowDPS}
\end{abstract}
\vspace{-5mm}

\section{Introduction}
The goal of  inverse problems is to recover the true signal $\vx_0$ from measurements $\vy$,
 which are generated through a forward model. This forward model typically represents a physical, computational, or statistical process that maps the true signal to the measurement space, often introducing noise, distortion, or information loss. 

For example, a typical linear inverse problem is defined as finding $\vx_0 \in \mathbb{R}^d$ from given measurement $ \vy$:
\begin{align}
    \vy = \rmA\vx_0 + \vn \in \mathbb{R}^{m} \ ,
    \label{eqn:ip}
\end{align}
where $\rmA \in \mathbb{R}^{m\times d}$ represents the forward model (or imaging system) and $\vn \in \mathbb{R}^{m} \sim \mathcal{N}(0, \sigma_n^2\rmI_m)$ denotes Gaussian noise with variance $\sigma_n$.
Because inverse problems are inherently ill-posed, there are  multiple solutions $\vx$ that satisfy \eqref{eqn:ip}. The key challenge, therefore, is to constrain the solution space to achieve a unique and meaningful reconstruction.

Traditionally, this has been done by adding a regularization term in a Maximum A Posteriori (MAP) framework \citep{boyd2011distributed}. More recently, advancements in diffusion models have introduced techniques for guiding the sampling trajectory to perform posterior sampling \cite{chung2022improving, chung2023diffusion, wang2023zeroshot, song2023pseudoinverseguided}. These diffusion model-based approaches aim to select the solution $\vx_0$ from the image distribution $q(\vx_0)$ (learned by the generative model) that is closest to the subspace defined by $\vy - \rmA \vx_0 = 0$. Such innovative methods have quickly become the state-of-the-art in the field, producing superior results compared to traditional approaches \cite{chung2022improving, chung2023diffusion, wang2023zeroshot, song2023pseudoinverseguided}.

Meanwhile, Latent Diffusion Models (LDMs) \cite{rombach2022high, patil2022stable} have emerged as foundational generative models trained on large-scale multi-modal datasets \cite{schuhmann2022laion}. By virtue of the flexibility and memory efficiency of the latent space, they can generate diverse content, including images, videos, audio, and multi-modal outputs simultaneously.
Building on the success of LDMs, recent zero-shot inverse problem solvers aim to utilize their superior generative prior to tackle inverse problems (\textit{e.g.}, PSLD \cite{rout2024solving}, ReSample \cite{song2024solving}, DAPS \cite{zhang2024improving}).
To further constrain the solution space, P2L \cite{chung2023prompt} leverages null-text embedding as a learnable parameter, and TReg \cite{kim2023regularization} introduces the concept of text-regularization for latent inverse problem solvers by using textual descriptions as guidance. Building on such ideas, DreamSampler \cite{kim2024dreamsampler} formulates an optimization problem in the latent space and proposes a text-guided image restoration algorithm. Such LDM-based solvers have proven effective in constraining the solution space even for inverse problems at higher scales.

Despite such notable progress in LDM-based inverse problem solvers, they often prove less effective when extended to the current dominant trend---flow-based models \cite{lipman2023flow, liu2023flow}. Flow-based models offer a robust foundation, and their integration into scalable architectures like the Diffusion Transformer has led to powerful generative capabilities \cite{peebles2023scalable, esser2024scaling}. Yet, 
{only a few flow-based inverse problem solvers have been proposed \cite{patel2024steering, martin2025pnpflow}. For example, FlowChef~\cite{patel2024steering} guides the velocity field using the gradient of the conditioning loss with respect to clean estimates computed at each point of the flow ODE under the assumption of local linear vector field and locally constant Jacobian. PnP-Flow~\cite{martin2025pnpflow} uses a pre-trained flow model as the denoiser for a plug-and-play restoration algorithm. However such existing methods do not provide enough insight into how they are related to inverse problem solving using posterior sampling.}

Therefore, one of the key contributions of this paper is providing a comprehensive guide to posterior sampling within the flow-based framework, addressing its subtleties and challenges. Building on these insights, we propose a novel approach that achieves unmatched reconstruction speed and quality, even for high-resolution images with severe degradations.
Specifically, we introduce {\em Flow-Driven Posterior Sampling (FlowDPS)}, a novel flow-based inverse problem solver derived from the decomposition of the flow ODE. 
More specifically, by deriving the flow-version of Tweedie's formula,  we demonstrate that the flow ODE can be broken down into two components---one for clean image estimation and another for noise estimation---and that manipulating both components significantly enhances reconstruction quality. 
Under the decomposition, we can easily see that
posterior sampling with data consistency is incorporated into the estimated clean image by integrating the likelihood gradient, 
while generative quality is maintained by adding stochastic noise to the noise component.
Consequently, we can easily show that the geometric structure of flow ODE is indeed similar to that of the diffusion model.
Furthermore,  we show that FlowDPS can be easily integrated into the state-of-the art latent flow models.

We validate FlowDPS on various linear inverse problems using the widely adopted flow-based Stable Diffusion 3.0 model \cite{esser2024scaling}.
Extensive experiments confirm that our method outperforms any existing LDM-based or flow-based solver. Overall, FlowDPS enables state-of-the-art reconstruction of high-resolution images for various inverse problems.

\section{Background}

\paragraph{Flow-based Models and Flow Matching}
Suppose that we have access to samples of source distribution $p$ and target distribution $q$ over $\mathbb{R}^d$.
The goal of a flow-based model is to generate $X_0 \sim q$, starting from $X_1 \sim p$.
Specifically, we define a time-dependent flow $\psi: (\vx, t) \rightarrow \psi_t(\vx)$ such that $\psi_t(X_1):=X_t\sim p_t$, where $p_t(\vx)$ with $0 \leq t \leq 1$ denotes a probability path with boundary conditions $p_1=p$ and $p_0=q$\footnote{While conventional notation uses $p_0=p$ and $p_1=q$, we swap the time index to align with the implementation of the flow-based model.}.
The flow can be uniquely defined by a flow ODE with velocity field $v_t$: 
\begin{align}
    \frac{d\psi_t(\vx)}{dt}=v_t(\psi_t(\vx)) \label{eqn:flow},  &\quad
   \mbox{where}\quad \psi_1(\vx)=\vx
\end{align}
Using the change of variable $\vx=\psi_t^{-1}(\vx')$, one can compute the corresponding velocity field $v_t(\vx')$ by 
\begin{align}\label{eq:v}
v_t(\vx')=v_t(\psi_t(\vx)) = \frac{d}{dt}\psi_t(\vx)=\dot\psi_t(\psi_t^{-1}(\vx')),
\end{align}
where $\dot\psi_t = d\psi_t/dt$.
When we train the flow model as a generative model, 
$v_t$ is learned using neural network parameterized by $\theta$   through flow matching loss \cite{lipman2023flow}:
\begin{align}
    \min_\theta \mathbb{E}_{t, \vx_t \sim p_t} \|v_t(\vx_t)-v_t^{\theta}(\vx_t)\|^2.
    \label{eqn:fm}
\end{align}
Unfortunately, a key problem to flow matching is that we cannot access $v_t(\vx_t)$ due to intractable integration over all $\vx_0$.
To address this, \cite{lipman2023flow} proposes conditional flow $\psi_t(\vx|\vx_0)$ that generate the probability 
$p_t(\vx|\vx_0)$.
Using \eqref{eq:v}, the conditional velocity field can be computed as
\begin{align}\label{eq:condv}
    v_t(\vx|\vx_0) = 
    \dot\psi_t(\psi_t^{-1}(\vx|\vx_0)|\vx_0),
\end{align}
and the parameterized velocity field is learned using the conditional flow matching loss,
\begin{align}
    \min_\theta \mathbb{E}_{t, \vx\sim p_{t|0}(\vx_t|\vx_0)} \|v_t(\vx_t|\vx_0)-v_t^{\theta}(\vx_t)\|^2,
    \label{eqn:cfm}
\end{align}
whose gradient with respect to $\theta$ is shown to match that of \eqref{eqn:fm} \cite{lipman2023flow} .
%

\paragraph{Affine conditional flows}
In practice, the flow model with affine coupling is one of the most widely used flow framework.
Specifically, 
let the source and target distributions form an independent coupling $\pi_{0,1}(\vx_0, \vx_1)=q(\vx_0)p(\vx_1)$ where $p(\vx_1)=\mathcal{N}(0,\rmI)$.
Now consider affine conditional flows:
\begin{align}
    \psi_t(\vx|\vx_0)=a_t\vx_0 + b_t\vx
    \label{eqn:affineflow}
\end{align}
with boundary conditions $a_0=1,b_0=0$ and $a_1=0,b_1=1$ such that
 $\psi_0(\vx|\vx_0)=\vx_0$ and $\psi_1(\vx|\vx_0)=\vx$.  
The forward path given condition $\vx_0$  is then given by
\begin{align}
    \vx_t = \psi_t(\vx_1|\vx_0)=a_t\vx_0+b_t \vx_1
    \label{eqn:forward}
\end{align}
Using \eqref{eq:condv} and \eqref{eqn:forward}, the conditional velocity at $\vx_t$ can be obtained by
\begin{align}
    v_t(\vx_t|\vx_0) =  \dot\psi_t(\vx_1|\vx_0)=\dot a_t \vx_0 + \dot b_t \vx_1. 
    \label{eqn:margin_v}
\end{align}
For the case of linear conditional flow, {where $a_t=1-t$ and $b_t=t$,} the conditional velocity is then given by
$v_t(\vx_t|\vx_0) = \dot\psi_t(\vx_t|\vx_0)= \vx_1-\vx_0$,
which leads to the popular form of the conditional flow matching loss:
\begin{align}
    \min_\theta \mathbb{E}_{t, \vx\sim p_{t|0}}\|v_t^{\theta}(\vx_t)-(\vx_1-\vx_0)\|^2,
    \label{eqn:cfm2}
\end{align}
where $\vx_t = (1-t)\vx_0+t\vx_1$. 
It is important to note that, despite utilizing conditional flow matching, the parameterized flow velocity estimate 
$v_t^{\theta}(\x_t)$ obtained through a neural network approximates the {\em marginal} velocity 
$v_t(\x_t)$ rather than the conditional velocity 
$v_t(\x_t|\x_0)$.
%

%
\section{Main Contribution} 
\subsection{Decomposition of Flow ODE}

Mathematically, the marginal velocity field $v_t(\vx)$ can be computed by \cite{lipman2023flow, liu2023flow}
\begin{align}
    v_t(\vx) 
    &=\mathbb{E}[v_t(\vx_t|\vx_0)|\vx_t=\vx] 
    = \mathbb{E}[\dot \psi_t (\vx_t|\vx_0) | \vx_t= \vx]  \notag\\
    &= \dot a_t\mathbb{E}[\vx_0|\vx_t=\x] + \dot b_t \mathbb{E}[\vx_1|\vx_t=\x] , 
    \label{eqn:margin_v_exp}
\end{align}
where we use $\dot \psi_t (\vx_t|\vx_0)=\dot a_t\vx_0+\dot b_t \vx_1$.
Here $\mathbb{E}[\vx_0|\vx_t]$ and $\mathbb{E}[\vx_1|\vx_t]$ corresponds to 
the denoised and noisy estimate for given $\vx_t$, respectively, which can
be computed using  the following flow-version of Tweedie formula whose proof can be found in Appendix.
\begin{restatable}[Tweedie Formula]{prop}{cmean}\label{prop:tweedie}
The denoised and noisy estimate given $\x_t$ are given by
\begin{align*}
\mathbb{E}[\vx_0|\vx_t] &= \left[a_t- \dot a_t\frac{ b_t }{\dot b_t}\right]^{-1}\left(\vx_t - \frac{b_t}{\dot b_t}v_t(\vx_t)\right)\\
\mathbb{E}[\vx_1|\vx_t] &= \left[b_t- \dot b_t\frac{a_t }{\dot a_t}\right]^{-1}\left(\vx_t - \frac{a_t}{\dot a_t}v_t(\vx_t)\right)
\end{align*}
\end{restatable}

This decomposition of the velocity through two Tweedie estimates leads to the decomposition of the flow ODE.
Specifically,   the flow ODE in \eqref{eqn:flow}  can be solved using the Euler method:
\begin{align}
    \vx_{t+dt} &= \vx_t + v_t(\vx_t)dt      \label{eqn:euler_step} \\
    &= C_1(t) \hat \vx_{0|t} + C_2(t)\hat \vx_{1|t}    \label{eqn:decompose}
\end{align}
where $dt <0$ denotes step size, $C_1(t)=a_t+\dot a_t dt$, $C_2(t)=b_t+\dot b_tdt$, $\hat \vx_{0|t}:=\mathbb{E}[\vx_0|\vx_t]$ and $\hat\vx_{1|t}:= \mathbb{E}[\vx_1|\vx_t]$. 
Note that the second term in \eqref{eqn:decompose} corresponds to deterministic noise in the literature of DDIM sampling.
Inspired by the merits of stochastic noise in the reverse sampling process, we propose to mix stochastic noise as
\begin{align}
    \vx_{t+dt}=C_1(t)\hat\vx_{0|t}+C_2(t)\tilde\vx_{1|t},
    \label{eqn:stochastic}
\end{align}
where 
\begin{align}\label{eq:mixs}
\tilde\vx_{1|t}=\sqrt{1-\eta_t}\hat\vx_{1|t} + \sqrt{\eta_t}\epsilonb,\quad \epsilonb \sim \mathcal{N}(0, \rmI ).
\end{align}
As shown in Section~\ref{sec:ddim}, this leads to 
    \begin{align}
        \vx_{t+dt}
        &=C_1(t)\hat\vx_{0|t} + \sqrt{C_2(t)^2-k_t^2}\hat\vx_{1|t} + \sqrt{k_t^2}\epsilonb
    \end{align}
    which is the same form with generalized version of DDIM \cite{song2020denoising} with  coefficients $C_1(t)=a_t+\dot a_tdt=a_{t+dt}$ and $C_2(t)=b_t+\dot b_tdt=b_{t+dt}$.

\subsection{Posterior Sampling via Flow Models}
\label{sec:posterior}
\noindent\textbf{Posterior flow velocity.}
Now we derive the flow velocity for a given posterior distribution.
Using the Tweedie formula in Proposition~\ref{prop:tweedie}, we get
\begin{align}
    v_t(\vx_t) 
    &= \frac{\dot a_t}{a_t}\vx_t + \left[\dot b_t -b_t\frac{\dot a_t}{a_t}\right] \mathbb{E}[\vx_1|\vx_t] \notag\\
    &= \frac{\dot a_t}{a_t}\vx_t - \left[\dot b_t b_t - b_t^2 \frac{\dot a_t}{a_t}\right] \nabla \log p_t(\vx_t), \label{eq:vtp}
\end{align}
where the second equality comes from \cite{lipman2023flow}
\begin{align}
    \nabla_{\vx_t}\log p_t(\vx_t) 
    &= \mathbb{E}_{\vx_0\sim q}[\nabla_{\vx_t} \log p_t(\vx_t|\vx_0)|\vx_t]\notag\\
    &=-\frac{1}{b_t}\mathbb{E}[\vx_1|\vx_t]\label{eqn:v_score}
\end{align}
since $\nabla_{\vx_t} \log p(\vx_t|\vx_0) = -\frac{1}{b_t^2}(\vx_t-a_t\vx_0)$ which arises from
$p_{t}(\vx_t|\vx_0)=\mathcal{N}(\vx_t|a_t \vx_0, b_t^2 \rmI)$.

Given that $v_t(\vx_t)$ in \eqref{eq:vtp} generates the probability path $p_t(\x_t)$,  we can easily
see that for the given measurement model \eqref{eqn:ip},  the velocity field conditioned on measurement $\vy$  given by
\begin{align}
    v_t(\vx_t|\vy) = \frac{\dot a_t}{a_t} \vx_t - \left[\dot a_t b_t - b_t^2 \frac{\dot a_t}{a_t}\right] \nabla_{\vx_t} \log p_t(\vx_t|\vy).
    \label{eqn:v_score_post}
\end{align}
generates the posterior sample path  according to $p_t(\x_t|\vy)$.
Using Bayes' rule, the score of the posterior distribution is
\begin{align}
   \nabla_{\vx_t} \log p_t(\vx_t|\vy) = \nabla_{\vx_t} \log p_t(\vy|\vx_t) + \nabla_{\vx_t} \log p_t(\vx_t).
\end{align}
Accordingly,  the velocity field is expressed by \cite{lipman2023flow}
\begin{align}
    v_t(\vx_t|\vy)=v_t(\vx_t) - \zeta_t \nabla_{\vx_t} \log p_t(\vy|\vx_t),
    \label{eqn:v_post}
\end{align}
where $\zeta_t=\dot a_t b_t - b_t^2 \frac{\dot a_t}{a_t}$. Thus, we can conduct posterior sampling by adding likelihood gradient to the original velocity field.

Unfortunately, the calculation of $\nabla_{\vx_t} \log p_t(\vy|\vx_t)$ is computationally prohibited
due to the integration with respect to $\vx_0$. 
Consequently, one could use the popular approximation technique in diffusion model called DPS
\cite{chung2023diffusion}:
    \begin{align}
  \nabla_{\vx_t} \log p_t(\vy|\vx_t)
        &\approx \nabla_{\vx_t} \log p(\vy|\mathbb{E}[\vx_0|\vx_t]) \label{eq:DPS}
    \end{align}
%
As \eqref{eq:DPS} requires backpropagation through neural networks, which could be often unstable,
one could use further approximation assuming piecewise linearity of the clean image manifold.
Specifically,  using the change of the variable, we have
%
%
\begin{align}
    \nabla_{\vx_t} \log p_t(\vy|\hat \vx_{0|t})=  J_\theta(\vx_t)\nabla_{\hat \vx_{0|t}} \log p(\vy|\hat \vx_{0|t}),
\end{align}
where $\hat \vx_{0|t}:=\mathbb{E}[\vx_0|\vx_t]$ refers to the denoised estimate and  
%
$J_\theta(\vx_t)=\partial \hat\vx_{0|t} / \partial \vx_t$ denotes the Jacobian.
The following proposition, which is a flow-version inspired by \cite{chung2024decomposed}, gives us an insight how to bypass
Jacobian computation.
%
\begin{restatable}{prop}{mcg}
    \label{prop:mcg}
    Suppose the clean data manifold $\mathcal{M}$ is represented as an affine subspace  and assume uniform distribution over $\mathcal{M}$. Then,
    \begin{align}
        J_\theta(\vx_t)=\frac{\partial \hat \vx_{0|t}}{\partial \vx_t}=\frac{1}{a_t}\mathcal{P}_{\mathcal{M}}
    \end{align}
    where $\mathcal{P}_{\mathcal{M}}$ denotes the orthogonal projection to $\mathcal{M}$.
\end{restatable}
\noindent
This implies that
if the standard gradient at the denoised estimate $\hat\vx_{0|t}$ does not leave
the clean manifold, we have
\begin{align}
\nabla_{\vx_t} \log p_t(\vy|\hat \vx_{0|t})&=\frac{1}{a_t}\mathcal{P}_{\mathcal{M}} \nabla_{\hat \vx_{0|t}} \log p(\vy|\hat \vx_{0|t})\notag\\
&= \frac{1}{a_t}\nabla_{\hat \vx_{0|t}} \log p(\vy|\hat \vx_{0|t}) 
\end{align}
Hence, the resulting posterior velocity field is given by
\begin{align}
    v_t(\vx_t|\vy)\approx v_t(\vx_t) - \frac{\zeta_t}{a_t} \nabla_{\hat \vx_{0|t}} \log p(\vy|\hat \vx_{0|t}),
\end{align}
where $\frac{\zeta_t}{a_t} = \dot a_t\frac{b_t}{a_t} (1 - \frac{b_t}{a_t})$. 

%
%

\begin{figure}
    \centering
    \includegraphics[width=\linewidth]{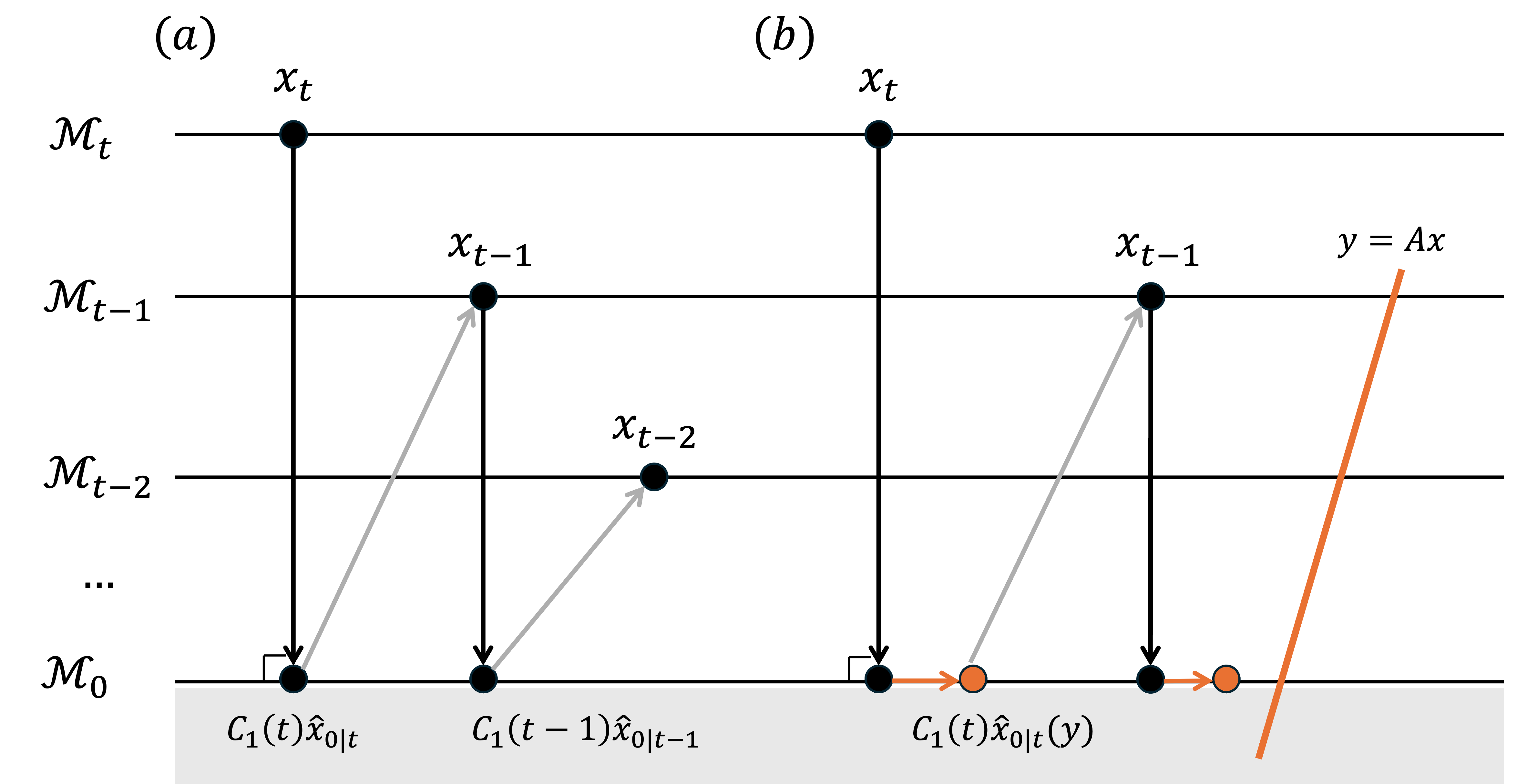}
    \caption{Geometry of FlowDPS. (a) Unconditional sampling of flow models where denoising and renoising are performed alternatively. (b) Posterior sampling of flow models where the data consistency offset is added to the denoised estimate. Orange arrow denotes the likelihood gradient in \eqref{eqn:likelihood_grad}.}
    \label{fig:concept}
\end{figure}

\paragraph{Posterior flow sampling.}
Finally, we solve the Euler method for flow ODE using the posterior flow velocity:
\begin{align}
    \vx_{t+dt} &= \vx_t + v_t(\vx_t|\vy) dt\nonumber\\
    &=\vx_t +\left( v_t(\vx_t) - \frac{\zeta_t}{a_t} \nabla_{\hat\vx_{0|t}} \log p(\vy|\hat\vx_{0|t})\right)dt\nonumber\\
    &=C_1(t)\tilde\vx_{0|t} + C_2(t) \hat\vx_{1|t},
    \label{eqn:euler_step_posterior}
\end{align}
where 
\begin{align}\label{eqn:tildetweedie}
\tilde \vx_{0|t} = \hat\vx_{0|t}-\beta_t\nabla_{\hat \vx_{0|t}} \log p(\vy|\hat \vx_{0|t})
\end{align} 
with $\beta_t := \frac{\zeta_t}{a_t}  \frac{dt}{C_1(t)}$. 
By inspecting ~\eqref{eqn:decompose} and \eqref{eqn:stochastic}, we can also use the following
DDIM version of posterior flow in an SDE form:
\begin{align}
    \vx_{t+dt} 
    &=C_1(t)\tilde\vx_{0|t} + C_2(t) \tilde\vx_{1|t},
    \label{eqn:euler_step_posterior_stochastic}
\end{align}
where $\tilde\vx_{1|t}$ is given by \eqref{eq:mixs}.
We call this general flow-based inverse problem solver as {\em Flow-Driven Posterior Sampling (FlowDPS)}. 
Fig.~\ref{fig:concept} depicts the geometric insight of FlowDPS, which shows striking similarity to the geometry of various diffusion posterior sampling methods.

\vspace*{0.1cm}
\noindent\textbf{Linear conditional flow.}
For the case of linear conditional flow, we have $a_t=1-t$ and $b_t=t$, leading to
$C_1(t) = 1-\sigma_{t+dt}$,
$C_2(t) = \sigma_{t+dt}$,  where $\sigma_{t}=t$.
Furthermore, the corresponding Tweedie formula becomes
\begin{align}
\hat \vx_{0|t} =\vx_t-t v_t(\vx_t),\quad
\hat \vx_{1|t}=\vx_t+(1-t)v_t(\vx_t)
\end{align}
 Then, FlowDPS can be implemented as 
\begin{align}
    \vx_{t+dt}= (1-\sigma_{t+dt}) \tilde \vx_{0|t} + \sigma_{t+dt} \tilde \vx_{1|t},
\end{align}
where $\tilde\vx_{0|t}$ and $\tilde \vx_{1|t}$ are from \eqref{eqn:tildetweedie} and \eqref{eq:mixs}.

\subsection{Comparison with other methods}
In this section, we discuss the improvements of FlowDPS compared to existing methods. FlowChef~\cite{patel2024steering} attempts to guide intermediate points of the flow ODE using the gradient of a loss function,
\begin{align}\label{eq:flowchef}
    \vx_t \leftarrow \vx_t - s \nabla_{\vx_{0|t}}\mathcal{L}(\rmA\hat\vx_{0|t}, \vy)
\end{align}
where $\mathcal{L}(\rmA\hat\vx_{0|t}, \vy)= \|\rmA\hat\vx_{0|t}-\vy\|^2$ and $s$ denotes constant step size.
For readers, \eqref{eq:flowchef} may seem similar to our FlowDPS at first glance. However, there are crucial differences between the two approaches.

While FlowChef applies a guiding gradient for $\vx_t$ as a whole, FlowDPS decomposes the flow ODE into $\vx_{0|t}$ and $\vx_{1|t}$, applying the gradient to the appropriate component $\vx_{0|t}$. Specifically, FlowDPS allows for inherent usage of adaptive step sizing with the scaling factor $\beta_t$, 
 and incorporation of stochastic noise-enhancements, which cannot be trivially introduced within the FlowChef framework.
As illustrated in Fig.~\ref{fig:beta} in the Appendix \ref{app:beta},
{$-\beta_t$ rapidly decreases to zero. This progression implies that our algorithm prioritizes likelihood maximization in the early denoising steps, aligning the sample with measurements.} Then, as denoising proceeds, the process gradually reverts to the unconstrained flow trajectory, compensating for any deviation introduced by the likelihood gradient.
This adaptive weighting of the guidance term is crucial. Empirically, we found that the overall structure of the sample is primarily determined in the early stages of the flow, while fine details are refined in the later stages. Thus, by dynamically adjusting the guidance influence, FlowDPS ensures a smooth transition from a measurement-consistent solution to a generative prior-driven process.
Most importantly, our approach establishes a clearer connection to posterior sampling, offering a more principled framework compared to prior methods \cite{patel2024steering, martin2025pnpflow}, which introduce a guiding vector to the ODE trajectory in a more heuristic manner.

\subsection{Latent FlowDPS}
\label{sec:latent_likelihood}

Similar to the score-based diffusion models, flow-based models can be defined in latent space. This enables flexible conditioning and reduces memory/time cost for generation. 
Specifically, let $\mathcal{E}_\phi$ and $\mathcal{D}_\varphi$ as the encoder and decoder that map samples between image and latent spaces, satisfying
\begin{align}
    \vz_0 = \mathcal{E}_\phi(\vx_0) = \mathcal{E}_\phi (\mathcal{D}_\varphi (\vz_0))
    \label{eqn:vae}
\end{align}
where $\vx_0$ and $\vz_0$ denote clean image and corresponding latent code, respectively.
As $\mathcal{E}_\phi$ and $\mathcal{D}_\varphi$ are fixed in our method, we discard subscript $\phi$ and $\varphi$ for brevity in the following sections.
By defining the target distribution $p_0$ as the distribution of latent code, conditional flow $\psi_t(\z|\z_0)$ and conditional velocity field $v_t(\z|\z_0)$ can be defined.
This allows us to solve problems at higher pixel resolutions, to leverage text conditions for the solution, and to leverage foundational generative models (such as various versions of Stable Diffusion).

\begin{figure*}[t]
    \centering
    \includegraphics[width=\linewidth]{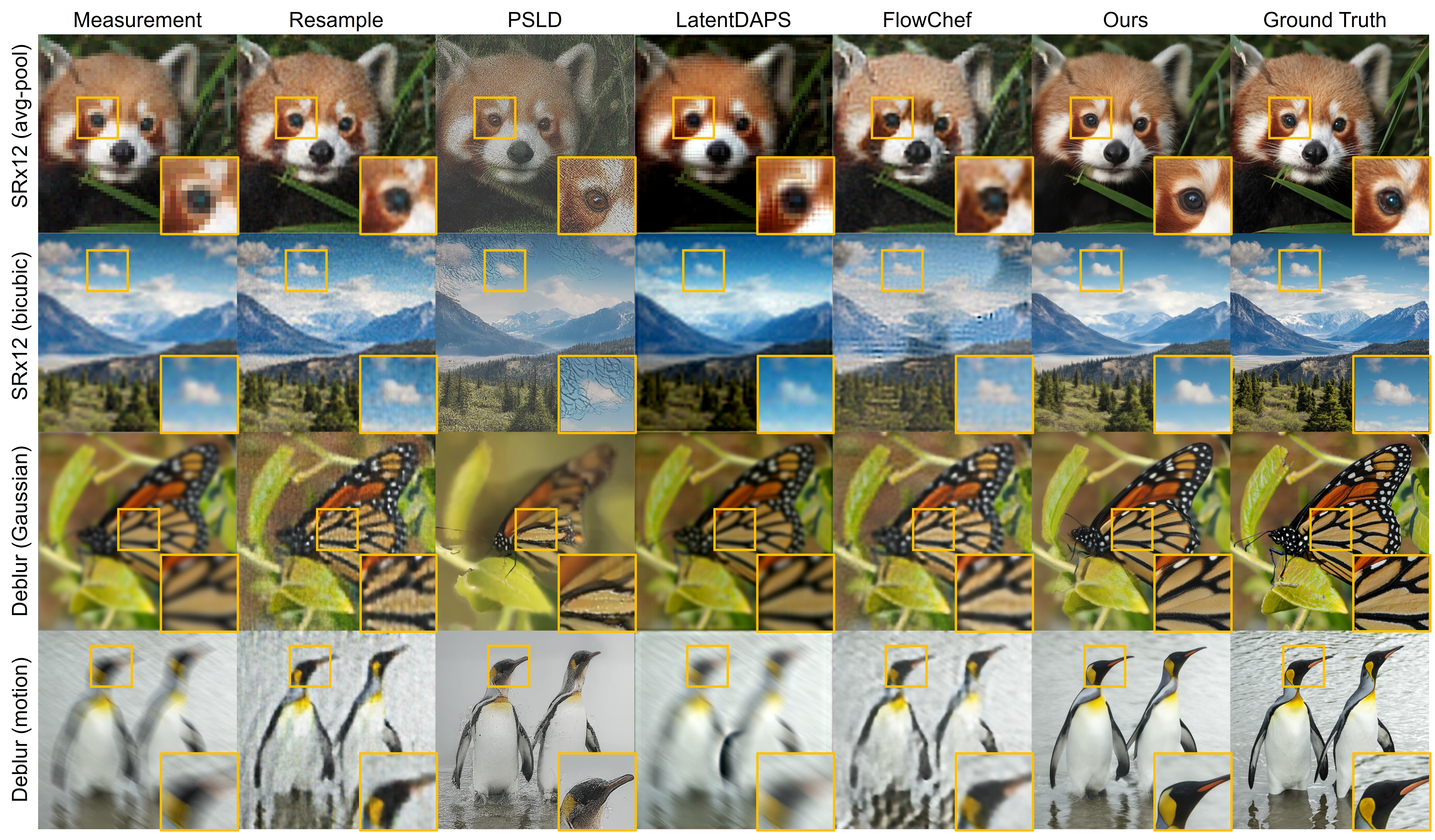}
    \caption{Qualitative comparison for linear inverse problems with DIV2K validation set. Insets show an enlarged view of the highlighted yellow boxes}
    \label{fig:qual_div2k}
\end{figure*}

Specifically, the denoised estimate of the latent code $\vz_t$ is expressed as
\begin{align}
    \tilde\vz_{0|t} = \hat\vz_{0|t}-\beta_t\nabla_{\hat\vz_{0|t}}\log p(\vy|\hat\vz_{0|t}).
    \label{eqn:latent_grad}    
\end{align}
The gradient is computed as
\begin{align}
    \nabla_{\hat \vz_{0|t}} \log p(\vy|\hat\vz_{0|t})
    &=\nabla_{\hat \vz_{0|t}} \log \int p(\vy|\vx_0)p(\vx_0|\hat\vz_{0|t})d\vx_0 \notag\\
    &\simeq \nabla_{\hat \vz_{0|t}} \log p(\vy|\mathcal{D}(\hat\vz_{0|t}))\label{eqn:likelihood_grad}
\end{align}
where the approximation comes from the assumption that $p(\vx_0|\vz_0):=\delta(\vx_0 - \mathcal{D}
(\vz_0))$ with decoder $\mathcal{D}$ for a near-perfect autoencoder.
As our forward model in \eqref{eqn:ip} includes Gaussian noise, we can model
\begin{align}
    p(\vy|\vx_0) = \mathcal{N}(\vy|\rmA\vx_0, \sigma_n^2 \rmI_m)
\end{align}
and the likelihood gradient \eqref{eqn:likelihood_grad} is computed by
\begin{align}
    \nabla_{\hat \vz_{0|t}} \log p(\vy|\mathcal{D}(\hat\vz_{0|t})) = - \nabla_{\hat\vz_{0|t}} \frac{\|\vy-\rmA\mathcal{D}(\hat\vz_{0|t})\|^2}{2\sigma_n^2}.
\end{align}
Building on the concept of Decomposed Diffusion Sampling (DDS) \cite{chung2023diffusion}, which employs multiple conjugate gradient steps to refine the denoised estimate, thereby accelerating convergence to the inverse solution while preserving updates within the affine space, we adopt a similar approach. Specifically, instead of a single gradient step for 
 $\tilde\vz_{0|t}$  as  in \eqref{eqn:latent_grad}, we perform multiple gradient steps to obtain 
$\hat\vz_{0|t}(\vy)$, enhancing the accuracy of the estimate.
To mitigate significant deviations of 
$\tilde \vz_{0|t}$ 
  from the flow model's sampling trajectory due to data consistency updates, we introduce an interpolation mechanism between the initial 
initial $\hat\vz_{0|t}$ and updated $\hat\vz_{0|t}(\vy)$. This is formulated as $\tilde \vz_{0|t} = (1-\gamma)\hat\vz_{0|t} + \gamma \hat\vz_{0|t}(\vy)$.
This interpolation ensures that the update remains aligned with the flow model's trajectory, preventing excessive divergence while maintaining consistency with the observed data.
%
%

\begin{table*}[t!]
    \centering
    \resizebox{\linewidth}{!}{
    \begin{tabular}{ccccccccccccccccc}
    \toprule
    \rowcolor{Gray}
    & \multicolumn{16}{c}{AFHQ 1k (768 $\times$ 768)}\\
    & \multicolumn{4}{c}{Super-resolution x12 (Avgpool)} & \multicolumn{4}{c}{Super-resolution x12 (Bicubic)} & \multicolumn{4}{c}{Deblurring (Gauss)} & \multicolumn{4}{c}{Deblurring (Motion)}\\
    \midrule
    Method &
    PSNR$\uparrow$ & SSIM$\uparrow$ & FID$\downarrow$ & LPIPS$\downarrow$ & PSNR$\uparrow$ & SSIM$\uparrow$ & FID$\downarrow$ & LPIPS$\downarrow$ & PSNR$\uparrow$ & SSIM$\uparrow$ & FID$\downarrow$ & LPIPS$\downarrow$ & PSNR$\uparrow$ & SSIM$\uparrow$ & FID$\downarrow$ & LPIPS$\downarrow$ \\
    \midrule
    ReSample &
    \underline{24.92} & \textbf{0.677} & 41.17 & 0.300 & \underline{24.94} & \textbf{0.676} & 39.94 & 0.297 & \underline{24.67} & \underline{0.665} & 44.22 & 0.325 & 24.97 & \underline{0.688} & 38.78 & 0.293 \\
    LatentDAPS &
    17.76 & 0.492 & 82.66 & 0.406 & 18.35 & 0.565 & 43.12 & 0.325 & 22.15 & 0.653 & \underline{27.89} & 0.280 & 21.61 & 0.670 & 52.28 & 0.315 \\
    PSLD &
    15.66 & 0.249 & 81.42 & 0.490 & 15.23 & 0.274 & 138.6 & 0.521 & 13.87 & 0.476 & 240.6 & 0.548 & 13.63 & 0.479 & 257.9 & 0.563 \\
    FlowChef &
    23.93 & 0.641 & \underline{21.14} & \underline{0.249} & 24.79 & 0.633 & \underline{21.31} & \underline{0.256} & 22.75 & \textbf{0.703} & 36.46 & \underline{0.267} & \textbf{26.16} & \textbf{0.732} & 39.19 & \underline{0.246} \\
    Ours &
    \textbf{25.88} & \underline{0.661} & \textbf{16.85} & \textbf{0.198} & \textbf{26.07} & \underline{0.662} & \textbf{15.71} & \textbf{0.188} & \textbf{25.33} & 0.653 & \textbf{23.00} & \textbf{0.238} & \underline{25.61} & 0.657 & \textbf{19.99} & \textbf{0.222} \\
    \midrule
    \rowcolor{Gray}
    & \multicolumn{16}{c}{FFHQ 1k (768 $\times$ 768)}\\
    & \multicolumn{4}{c}{Super-resolution x12 (Avgpool)} & \multicolumn{4}{c}{Super-resolution x12 (Bicubic)} & \multicolumn{4}{c}{Deblurring (Gauss)} & \multicolumn{4}{c}{Deblurring (Motion)}\\
    \midrule
    Method &
    PSNR$\uparrow$ & SSIM$\uparrow$ & FID$\downarrow$ & LPIPS$\downarrow$ & PSNR$\uparrow$ & SSIM$\uparrow$ & FID$\downarrow$ & LPIPS$\downarrow$ & PSNR$\uparrow$ & SSIM$\uparrow$ & FID$\downarrow$ & LPIPS$\downarrow$ & PSNR$\uparrow$ & SSIM$\uparrow$ & FID$\downarrow$ & LPIPS$\downarrow$ \\
    \midrule
    ReSample &
    25.09 & \textbf{0.734} & 102.7 & 0.304 & 25.09 & \textbf{0.732} & 102.4 & 0.301 & \underline{24.95} & \underline{0.724} & 107.4 & 0.333 & \underline{25.39} & \underline{0.745} & 95.16 & 0.303 \\
    LatentDAPS &
    18.09 & 0.541 & 128.8 & 0.368 & 18.62 & 0.614 & 105.1 & 0.294 & 22.56 & 0.714 & \underline{72.97} & \underline{0.248} & 22.06 & 0.722 & \underline{81.63} & 0.287 \\
    PSLD &
    15.05 & 0.240 & 142.1 & 0.471 & 15.20 & 0.275 & 141.6 & 0.472 & 12.98 & 0.519 & 265.6 & 0.563 & 12.98 & 0.527 & 268.1 & 0.575 \\
    FlowChef &
    \underline{25.40} & 0.696 & \underline{41.50} & \underline{0.218} & \underline{25.31} & 0.692 & \underline{39.75} & \underline{0.224} & 23.14 & \textbf{0.767} & 112.8 & 0.251 & 23.14 & \textbf{0.772} & 104.7 & \underline{0.264} \\
    Ours &
    \textbf{26.81} & \underline{0.703} & \textbf{33.78} & \textbf{0.159} & \textbf{27.00} & \underline{0.702} & \textbf{33.75} & \textbf{0.154} & \textbf{26.23} & 0.702 & \textbf{41.31} & \textbf{0.197} & \textbf{26.62} & 0.699 & \textbf{38.14} & \textbf{0.182} \\
    \midrule
    \rowcolor{Gray}
    & \multicolumn{16}{c}{DIV2K 0.8k (768 $\times$ 768)}\\
    & \multicolumn{4}{c}{Super-resolution x12 (Avgpool)} & \multicolumn{4}{c}{Super-resolution x12 (Bicubic)} & \multicolumn{4}{c}{Deblurring (Gauss)} & \multicolumn{4}{c}{Deblurring (Motion)}\\
    \midrule
    Method &
    PSNR$\uparrow$ & SSIM$\uparrow$ & FID$\downarrow$ & LPIPS$\downarrow$ & PSNR$\uparrow$ & SSIM$\uparrow$ & FID$\downarrow$ & LPIPS$\downarrow$ & PSNR$\uparrow$ & SSIM$\uparrow$ & FID$\downarrow$ & LPIPS$\downarrow$ & PSNR$\uparrow$ & SSIM$\uparrow$ & FID$\downarrow$ & LPIPS$\downarrow$ \\
    \midrule
    ReSample &
    \underline{20.41} & \textbf{0.505} & \underline{101.5} & \underline{0.352} & \underline{20.48} & \textbf{0.504} & \underline{99.73} & \underline{0.349} & \underline{20.31} & 0.493 & 113.7 & 0.386 & \underline{20.82} & 0.530 & 96.73 & 0.343 \\
    LatentDAPS &
    14.34 & 0.356 & 183.3 & 0.412 & 16.92 & 0.437 & 104.7 & 0.383 & 19.38 & \underline{0.503} & \underline{83.72} & 0.346 & 19.20 & \underline{0.542} & 121.7 & 0.380 \\
    PSLD &
    13.90 & 0.187 & 124.0 & 0.523 & 14.01 & 0.207 & 118.1 & 0.519 & 14.92 & 0.313 & 97.20 & 0.527 & 14.54 & 0.307 & 108.3 & 0.547 \\
    FlowChef &
    18.02 & \underline{0.495} & \underline{101.5} & 0.355 & 15.12 & 0.441 & 123.4 & 0.433 & 19.54 & \textbf{0.520} & 98.27 & \underline{0.343} & 19.80 & \textbf{0.546} & \underline{88.86} & \underline{0.313}\\
    Ours &
    \textbf{20.84} & 0.488 & \textbf{58.29} & \textbf{0.258} & \textbf{21.10} & \underline{0.492} & \textbf{56.19} & \textbf{0.252} & \textbf{20.46} & 0.473 & \textbf{72.13} & \textbf{0.319} & \textbf{21.04} & 0.495 & \textbf{62.68} & \textbf{0.291} \\
    \bottomrule
    \end{tabular}
    }
    \caption{Quantitative comparison for linear inverse problems. For each dataset, we evaluate inverse problem solvers using 1k, 1k, and 0.8k image sets of AHFQ, FFHQ and DIV2K. \textbf{Bold}: the best, \underline{Underline}: the second best.}
    \label{tab:quantitative}
    \vspace{-5mm}
\end{table*}

Specifically, we devise our algorithm by setting the values $\gamma=\sigma_t$ and $\eta=1-\sigma_{t+dt}$.
This is to imitate the behavior of $\beta_t$ in \eqref{eqn:tildetweedie}, which monotonically decreases to zero.
Specifically, the selection $\gamma=\sigma_t$ puts emphasis on $\hat\vz_{0|t}(\vy)$ in early steps, guiding $\tilde\vz_{0|t}$ towards a more data-consistent path. $\hat\vz_{0|t}$ is weighted more highly in later steps for better generation of high-frequency details.
Similarly, stochastic noise $\epsilonb$ is added under $\eta=1-\sigma_{t+dt}$ to use a higher level of stochastic noise in later steps and a higher level of deterministic noise $\hat\vz_{1|t}$ in early steps. 
Further analysis regarding our selections is provided in Sec.~\ref{sec:ablation}. The overall algorithm is described in Algorithm~\ref{alg:FlowDPS}.

\begin{algorithm}
\caption{Algorithm of {FlowDPS} (SD3.0, FLUX)}\label{alg:FlowDPS}
\begin{algorithmic}[1]
\Require Measurement $\vy$, Linear operator $\rmA$, Pre-trained flow-based model $\vv_\theta$, VAE encoder and Decoder $\mathcal{E}, \mathcal{D}$, Text embeddings $c_\varnothing, c$, CFG scale $\lambda$, Stochasticity level $\eta$, Noise Schedule $\sigma_t$
\State $\vz \sim \mathcal{N}(0, \rmI_d)$
\For{$t: 1\rightarrow 0$}
    \State $\vv_t(\vz) \gets \vv_\theta(\vz, c_\varnothing) + \lambda (\vv_\theta(\vz, c) - \vv_\theta(\vz, c_\varnothing))$
    \State $\hat \vz_{0|t} \gets \vz - \sigma_t\vv_t(\vz)$
    \State $\hat \vz_{1|t} \gets \vz + (1-\sigma_t)\vv_t(\vz)$
    \\\Comment{1. Likelihood Gradient}
    \State $\hat \vz_{0|t}(\vy) \gets \argmin_{\vz} \| \vy - \rmA \mathcal{D}(\vz) \|^2$.  
    \State $\tilde \vz_{0|t} \gets \sigma_t \hat \vz_{0|t}(\vy) + (1-\sigma_t) \hat \vz_{0|t}$
    \\\Comment{2. Stochasticity}
    \State $\epsilon \sim \mathcal{N}(0, \rmI_d)$
    \State $\tilde \vz_{1|t} \gets \sqrt{\sigma_{t+dt}}\vz_{1|t} + \sqrt{1-\sigma_{t+dt}}\epsilon$
    \\\Comment{3. Euler update}
    \State $\vz \gets (1-\sigma_{t+dt})\tilde \vz_{0|t} + \sigma_{t+dt} \tilde \vz_{1|t}$
\EndFor
\end{algorithmic}
\end{algorithm}

\section{Experiments}
\subsection{Experimental setup}

\noindent\textbf{Datasets and evaluation metrics.}
To demonstrate the performance of the proposed method in various domains, we use three datasets: 1k validation images of AFHQ, FFHQ, and 0.8k training images of DIV2K. We set the image resolution to $768 \times 768$ by resizing the original dataset.
Quantitative evaluation is conducted on PSNR, SSIM for pixel-level fidelity, and FID, LPIPS for perceptual quality.

\noindent\textbf{Baselines.}
We compare our method with PSLD \cite{rout2024solving}, LatentDAPS \cite{zhang2024improving}, Resample \cite{song2024solving} (recent latent diffusion-based inverse problem solvers), and FlowChef \cite{patel2024steering} (a conditioning framework designed for flow-based models). To ensure fair comparison, we use the pre-trained Stable Diffusion 3.0 - medium \cite{esser2024scaling} as the backbone model for all solvers. Further details are provided in the Appendix. The results of FLUX version of FlowDPS can be also found in Appendix.

\noindent\textbf{Text conditioning.}
We leverage text prompts as conditions for the latent flow-based model to refine the solution. Specifically, we use the prompt ``a photo of a closed face of a dog (cat)" for AFHQ, and the prompt ``a photo of a closed face" for FFHQ.
Notably, defining an appropriate text prompt solely based on the given measurement can be challenging in practical scenarios.
For instance, the text prompts for the DIV2K dtaset must include more details compared to AFHQ or FFHQ, since each image may contain multiple  objects.
To evaluate effectiveness in such real-world scenarios, we utilize text prompts extracted from the measurement using DAPE \cite{wu2024seesr} as a degradation-aware text prompt extractor. For all cases, we set CFG \cite{ho2021classifierfree} scale to 2.0.

\subsection{Experimental Results with Inverse Problems}
\noindent\textbf{Noisy linear inverse problems.}
FlowDPS is a general inverse problem solver that does not require additional training, thus we demonstrate its broad applicability across four linear inverse problems. Specifically, we consider i) Super-resolution from average pooling with scale factor of 12, ii) Super-resolution from bicubic interpolation with scale factor of 12, iii) Gaussian deblurring with kernel size of 61 and standard deviation of 3.0, and iv) Motion deblurring with kernel size of 61 and intensity value of 0.5. In all tasks, we add Gaussian noise to the measurement with $\sigma_n=0.03$.

\begin{figure*}[t]
    \centering
    \includegraphics[width=\linewidth]{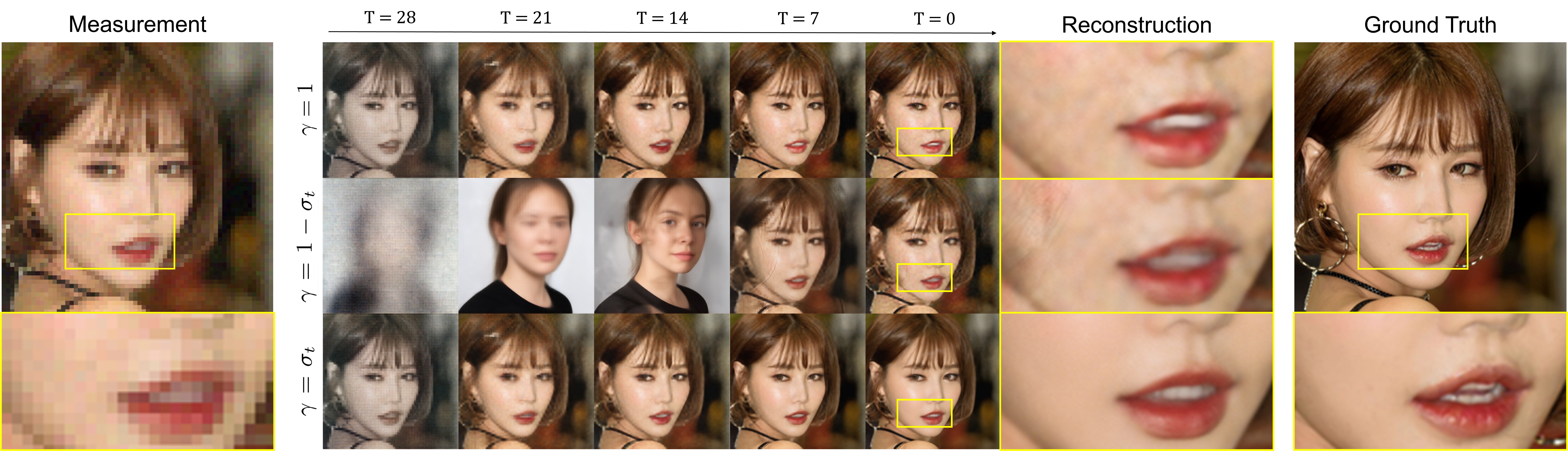}
    \caption{Sampling trajectories for various choices of interpolation scale $\gamma$. Insets show an enlarged view of the highlighted yellow boxes.}
    \label{fig:abl_gamma}
    \vspace{-3mm}
\end{figure*}

\begin{figure}[t]
    \centering
    \includegraphics[width=\linewidth]{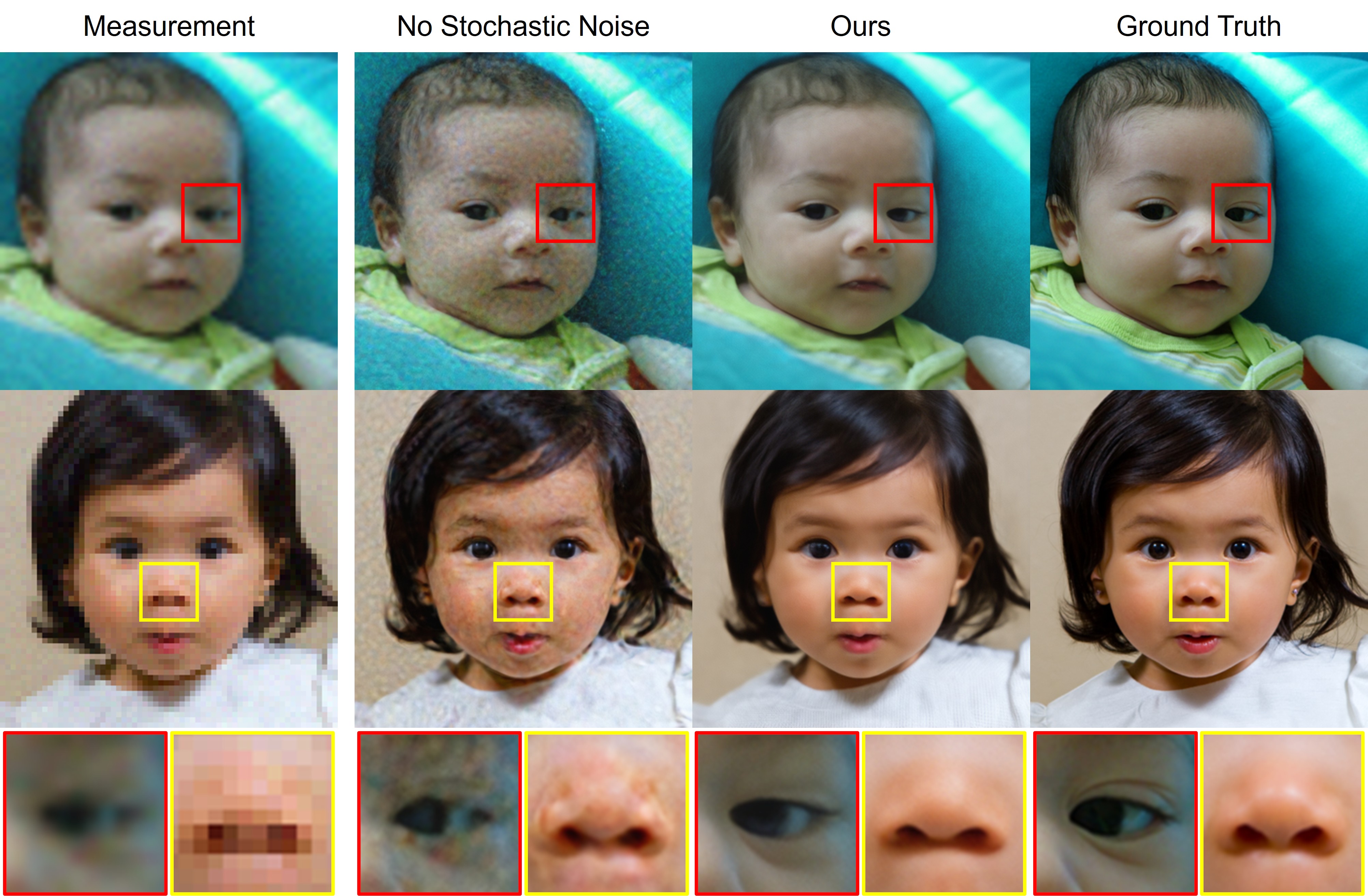}
    \caption{Qualitative results regarding ablation of stochastic noise.}
    \label{fig:abl_noise}
\end{figure}

\begin{table}[t]
    \centering
    \resizebox{\linewidth}{!}{
        \begin{tabular}{c cccc cccc}
        \toprule
        & \multicolumn{4}{c}{Super-resolution x12 (Avgpool)} & \multicolumn{4}{c}{Super-resolution x12 (Bicubic)} \\
        \midrule
        AFHQ &
        PSNR$\uparrow$ & SSIM$\uparrow$ & FID$\downarrow$ & LPIPS$\downarrow$ & PSNR$\uparrow$ & SSIM$\uparrow$ & FID$\downarrow$ & LPIPS$\downarrow$ \\
        \midrule
        Ours (w/o $\epsilon$) &
        24.75 & 0.609 & 32.85 & 0.251 & 24.80 & 0.602 & 34.81 & 0.268 \\
        Ours (w/ $\epsilon$) &
        \textbf{25.88} & \textbf{0.661} & \textbf{16.85} & \textbf{0.198} & \textbf{26.07} & \textbf{0.662} & \textbf{15.71} & \textbf{0.188}\\
        \midrule
        FFHQ &
        PSNR$\uparrow$ & SSIM$\uparrow$ & FID$\downarrow$ & LPIPS$\downarrow$ & PSNR$\uparrow$ & SSIM$\uparrow$ & FID$\downarrow$ & LPIPS$\downarrow$ \\
        \midrule
        Ours (w/o $\epsilon$) &
        25.61 & 0.668 & 78.67 & 0.249 & 25.53 & 0.656 & 82.07 & 0.275 \\
        Ours (w/ $\epsilon$) &
        \textbf{26.81} & \textbf{0.703} & \textbf{33.78} & \textbf{0.159} & \textbf{27.00} & \textbf{0.702} & \textbf{33.75} & \textbf{0.154} \\
        \bottomrule
        \end{tabular}
    }
    \caption{Ablation of stochastic noise on AFHQ and FFHQ images.}
    \label{tab:abl_noise}
\vspace{-5mm}
\end{table}

\noindent\textbf{Comparison results.}
Quantitative and qualitative comparison results are shown in Table \ref{tab:quantitative} and Figure \ref{fig:qual_div2k}. Results show that FlowDPS outperforms existing methods across all datasets for various linear inverse problems. In particular, inverse problem solvers designed for latent diffusion-based models (\textit{i.e.}, ReSample, LatentDAPS, PSLD) perform suboptimally when applied to the linear flow-based Stable Diffusion 3.0.
These methods yield images that partially resolve the inverse problem, but their outputs remain significantly degraded---an anticipated challenge when reconstructing highly degraded, high-resolution (768×768) images. Thus, blurry results cause pixel-fidelity results (\textit{e.g.}, SSIM) of baseline methods to be abnormally high in some cases as reported in \cite{ledig2017photo, zhang2018unreasonable}, but substandard perceptual results (\textit{e.g.}, FID, LPIPS) signify the suboptimal performance. Consequently, inverse problem solvers devised for LDMs cannot be naively applied in flow-based settings, supporting the pressing need for a suitable framework.

In contrast to LDM-based solvers, methods specially developed for linear flow-based models (\textit{i.e.}, FlowChef, FlowDPS) yield significantly improved performance regarding the metrics evaluated in Table \ref{tab:quantitative}. However, qualitative results show how FlowChef fails to achieve acceptable visual quality in challenging degradation scenarios. In contrast, FlowDPS consistently produces state-of-the-art results, regarding both image fidelity and perceptual quality. {More qualitative results, including higher resolutions and FLUX~\cite{flux2024} implementation are provided in Appendix \ref{app:additional}.}

\subsection{Ablation study}
\label{sec:ablation}
To further validate the effectiveness of our method, we perform various ablation studies regarding important selections: interpolation parameter $\gamma$ and amount of stochastic noise $\eta$. Also, we explore the effect of NFE and CFG scale.

\noindent\textbf{Interpolation parameter $\gamma$.}
In FlowDPS, we manipulate the clean image estimation to balance data consistency and sample quality. Specifically, the initial $\hat\vz_{0|t}$ and updated $\hat\vz_{0|t}(\vy)$ is interpolated based on an interpolation scale $\gamma$ such that $\tilde \vz_{0|t} = (1-\gamma)\hat\vz_{0|t} + \gamma \hat\vz_{0|t}(\vy)$. Note that choice of parameter $\gamma$ is flexible, and in Figure \ref{fig:abl_gamma} we demonstrate the sampling trajectories of certain examples. Selecting $\gamma=1$ (\textit{i.e.}, using the updated $\hat\vz_{0|t}(\vy)$ directly) causes deviation of $\tilde \vz_{0|t}$ from the sampling trajectory of the base model, causing reconstructions to become blurry. Selecting $\gamma=1-\sigma_t$ puts emphasis on $\hat\vz_{0|t}$ in the early stages of denoising and $\hat\vz_{0|t}(\vy)$ in the later stages, confusing the model in which sampling trajectory it should take. Since overall structure is determined early on in the denoising process (see Figure \ref{fig:abl_gamma}), it is reasonable to weight $\hat\vz_{0|t}(\vy)$ highly in the initial steps to generate data-consistent structure, and weight $\hat\vz_{0|t}$ highly in the final steps to better produce high-frequency fine details. Thus, in our final algorithm we select $\gamma=\sigma_t$.

\noindent\textbf{Importance of stochastic noise.}
We provide quantitative and qualitative analysis regarding the importance of stochastic noise in Figure \ref{fig:abl_noise} and Table \ref{tab:abl_noise}. Specifically, we observe that using only deterministic noise causes artifacts to appear, and that mixing stochastic noise removes such artifacts while guiding samples towards high-quality results.
%


\noindent\textbf{Relation between NFE and CFG scale $\lambda$.}
NFE and CFG scale $\lambda$ are important factors for generating high-quality results that contain high-frequency details. We find using $\text{NFE}=28$ and $\lambda=2.0$ to be sufficient for producing high-quality results.
Further discussion regarding the relation between NFE and CFG scale is provided in Appendix \ref{app:cfgnfe}.

%
%
%


\vspace*{-0.1cm}
\section{Conclusion}
\vspace*{-0.2cm}
This work addresses the pressing need for a general flow-based inverse problem solver, bridging a gap that LDM-based approaches have yet to fill. In order to draw on the powerful generative capacity of flow-based models, we presented Flow-Driven Posterior Sampling (FlowDPS), which carefully decomposes the flow ODE into clean image and noise components. Through targeted manipulation of each component, FlowDPS seamlessly integrates data consistency and preserves generative quality to yield superior results on various linear inverse problems even for high-resolution images with severe degradations.
More broadly, the general framework we introduce clarifies the process of posterior sampling in flow-based generative models, offering both a theoretical perspective and a practical method. We anticipate that this approach will inspire further work 
across a wide range of applications.

{
    \small
    \bibliographystyle{ieeenat_fullname}
    \bibliography{main}
}

\clearpage
\setcounter{page}{1}
\maketitlesupplementary

\section{Derivation of DDIM form}
\label{sec:ddim}

    The stochastic part of the equation \eqref{eqn:stochastic} is
    \begin{align}
        &C_2(t) \tilde \vx_{1|t}\nonumber\\
        &= C_2(t) (\sqrt{1-\eta_t}\hat\vx_{1|t} + \sqrt{\eta_t}\epsilonb)\nonumber\\
        &= \sqrt{C_2(t)^2-C_2(t)^2\eta_t}\hat\vx_{1|t} + \sqrt{C_2(t)^2\eta_t}\epsilonb.
    \end{align}
    If we set $k_t = C_2(t)\sqrt{\eta_t}$, this is equivalent to
    \begin{align}
        \sqrt{C_2(t)^2-k_t^2}\hat\vx_{1|t} + \sqrt{k_t^2}\epsilonb
    \end{align}
    and the one step update is expressed as
    \begin{align}
        &\vx_{t+dt}\nonumber\\
        &=C_1(t)\hat\vx_{0|t} + \sqrt{C_2(t)^2-k_t^2}\hat\vx_{1|t} + \sqrt{k_t^2}\epsilonb
    \end{align}
    which is in the same form as DDIM \cite{song2020denoising} but with different coefficients $C_1(t)=a_t+\dot a_tdt=a_{t+dt}$ and $C_2(t)=b_t+\dot b_tdt=b_{t+dt}$.

    

\section{Proofs}

\cmean*
\begin{proof}
According to \eqref{eqn:margin_v}, we have $ v_t(\vx_t|\vx_0) = \dot a_t \vx_0 + \dot b_t \vx_1.$
This leads to the representation of $\x_t$ with respect to the conditional flow velocity:
\begin{align}
    \vx_t &= a_t\vx_0 + b_t \vx_1 \notag\\
    &= a_t\vx_0 + b_t\left(\frac{v_t(\vx_t|\vx_0)-\dot a_t \vx_0}{\dot b_t}\right) \notag\\
    &= \left(a_t- \dot a_t\frac{ b_t }{\dot b_t}\right)\x_0 +\frac{b_t}{\dot b_t} v_t(\vx_t|\vx_0)
\end{align}
Thus we have
\begin{align*}
    \mathbb{E}[\vx_0|\vx_t] &= \left(a_t- \dot a_t\frac{ b_t }{\dot b_t}\right)^{-1}\left(\vx_t - \frac{b_t}{\dot b_t}\mathbb{E}[v_t(\vx_t|\vx_0)]\right) \\
    &=\left(a_t- \dot a_t\frac{ b_t }{\dot b_t}\right)^{-1}\left(\vx_t - \frac{b_t}{\dot b_t}v_t(\vx_t)\right)
   \end{align*}
Similarly, we have
\begin{align}
    \vx_t &= a_t\vx_0 + b_t \vx_1 \notag\\
    &= a_t \left(\frac{v_t(\vx_t|\vx_0)-\dot b_t \vx_1}{\dot a_t}\right)   + b_t \vx_1\notag\\
    &= \left(b_t- \dot b_t\frac{ a_t }{\dot a_t}\right)\x_1 +\frac{a_t}{\dot a_t} v_t(\vx_t|\vx_0)
\end{align}
which leads to
\begin{align*}
    \mathbb{E}[\vx_1|\vx_t] &= \left(b_t- \dot b_t\frac{a_t }{\dot a_t}\right)^{-1}\left(\vx_t - \frac{a_t}{\dot a_t}\mathbb{E}[v_t(\vx_t|\vx_0)]\right) \\
    &=  \left(b_t- \dot b_t\frac{a_t }{\dot a_t}\right)^{-1}\left(\vx_t - \frac{a_t}{\dot a_t}v_t(\x_t)\right)
   \end{align*}

\end{proof}

\mcg*
\begin{proof}
    Using \eqref{eqn:margin_v_exp} and \eqref{eqn:v_score}, we can express the Jacobian as
    \begin{align}
        J_\theta(\vx) = \frac{1}{a_t}\frac{\partial (\vx_t - b_t^2 \nabla_{\vx_t} \log p(\vx_t))}{\partial \vx_t}.
    \end{align}
    Here, we will derive the score function $\nabla_{\vx_t} \log p(\vx_t)$ in a closed form solution with assumptions and prove the result.
    
    From the definition of affine conditional flow,
    \begin{align}
        \vx_t = \psi_t(\vx_1|\vx_0) = a_t \vx_0 + b_t \vx_1,
    \end{align}
    where $\vx_1 \sim \mathcal{N}(0, \rmI_d)$,
    we get the explicit expression of
    \begin{align}
        p(\vx_t|\vx_0) = \frac{1}{(2\pi b_t^2)^{d/2}} \exp \left(-\frac{\|\vx_t - a_t\vx_0\|^2}{2b_t^2}\right).
    \end{align}
    Assume that the clean images are distribution on subspace $\mathcal{M}$ uniformly. To express this, we start from defining $p(\vx_0)$ as a zero-mean Gaussian distribution with isotropic variance $\sigma$.
    \begin{align}
        p(\vx_0) = \frac{1}{(2\pi \sigma^2)^{l/2}} \exp \left(-\frac{\|\mathcal{P}_\mathcal{M} \vx_0 \|^2}{2\sigma^2}\right)
        \label{eqn:uniform}
    \end{align}
    as $\mathcal{P}_{M}\vx_0 = \vx_0$. 
    Considering the marginal density
    \begin{align}
        p(\vx_t) = \int p(\vx_t|\vx_0)p(\vx_0) d\vx_0,
    \end{align}
    we have to compute $p(\vx_t|\vx_0)p(\vx_0)$ that is
    \begin{align}
        p(\vx_t|\vx_0)p(\vx_0) = \frac{1}{(2\pi b_t^2)^{d/2}(2\pi \sigma^2)^{l/2}} \exp (-d(\vx_t, \vx_0)),
    \end{align}
    where
    \begin{align}
        &d(\vx_t, \vx_0) = \frac{\|\vx_t - a_t \vx_0\|^2}{2b_t^2} + \frac{\| \mathcal{P}_\mathcal{M}\vx_0 \|^2}{2\sigma^2}\nonumber\\
        &= \frac{\|\mathcal{P}_\mathcal{M}^\perp \vx_t\|^2}{2b_t^2} + \frac{\|\mathcal{P}_\mathcal{M}\vx_t - a_t \vx_0\|^2}{2b_t^2} + \frac{\| \mathcal{P}_\mathcal{M}\vx_0 \|^2}{2\sigma^2}\nonumber \\
        &= \frac{\|\mathcal{P}_\mathcal{M}^\perp\vx_t\|^2 - c_t \|\mathcal{P}_\mathcal{M}\vx_t\|^2}{2b_t^2} + \frac{\| \mathcal{P}_\mathcal{M}\vx_0 - \frac{1-c_t}{a_t}\mathcal{P}_\mathcal{M}\vx_t \|^2}{\sigma^2c_t},\nonumber 
    \end{align}
    and
    \begin{align}
        c_t = \frac{b_t^2}{b_t^2 + \sigma^2+a_t^2}.
    \end{align}
    Therefore, after integrating out with respect to $\vx_0$, we have
    \begin{align}
        \log p(\vx_t) = -\frac{\|\mathcal{P}_\mathcal{M}^\perp\vx_t\|^2 - c_t \|\mathcal{P}_\mathcal{M}\vx_t\|^2}{2b_t^2} + const.,
    \end{align}
    leading to 
    \begin{align}
        \nabla_{\vx_t} \log p(\vx_t) = -\frac{\mathcal{P}_\mathcal{M}^\perp\vx_t - c_t\mathcal{P}_\mathcal{M}\vx_t}{b_t^2}.
    \end{align}
    By the assumption of uniform distribution in \eqref{eqn:uniform} with $\sigma \rightarrow \infty$, we have $c_t \rightarrow 0$. Therefore,
    \begin{align}
        \lim_{\sigma \rightarrow \infty} \nabla_{\vx_t} \log p(\vx_t) = -\frac{1}{b_t^2}\mathcal{P}_\mathcal{M}^\perp \vx_t
    \end{align}
    and we conclude that
    \begin{align}
        J_\theta(\vx_t) &= \frac{1}{a_t}\frac{\partial (\vx_t - b_t^2\nabla_{\vx_t}\log p(\vx_t))}{\partial \vx_t}\\
        &= \frac{1}{a_t}\frac{\partial (\vx_t - \mathcal{P}_\mathcal{M}^\perp \vx_t)}{\partial \vx_t}\\
        &= \frac{1}{a_t}\frac{\partial \mathcal{P}_\mathcal{M} \vx_t}{\partial \vx_t} = \frac{1}{a_t} \mathcal{P}_\mathcal{M}.
    \end{align}
\end{proof}

\section{Implementation details}
In this section, we provide implementation details of the FlowDPS and baselines. The code will be released to public at \url{https://github.com/FlowDPS-Inverse/FlowDPS}.

\paragraph{Implementation with flow models} For a fair comparison, we re-implement baselines that are proposed with score-based diffusion models. In the following, we provide details for each implementation. 

\noindent
\textbf{PSLD~\cite{rout2024solving}} {extends DPS~\cite{chung2023diffusion} to latent diffusion models by introducing a novel regularization loss for the autoencoder. Since its gradient is computed with respect to the intermediate sample $\vx_t$ during the reverse diffusion process, the same algorithm can be implemented using Euler’s method without loss of generality.}

\noindent
\textbf{ReSample}
{incorporates data consistency into the reverse sampling process of LDMs by solving an optimization problem on some time steps. The key ideas of this framework is Stochastic Resampling, for renoising the optimized latent, and hard data consistency. We implement Resample for linear flow-based models by setting $\bar\alpha_t=\frac{(1-\sigma_t)^2}{\sigma_t^2+(1-\sigma_t)^2}$. Resample uses various techniques such as dividing the sampling process into three stages and separately using soft data consistency and hard data consistency. We find that adopting this same setting for our comparison produces extremely poor results, primarily because of the small number of ODE steps of the Euler solver. We empirically find that using skip step size of 1 to perform hard data consistency on all steps produces best results, and use this setting for comparison.
}

\noindent
\textbf{LatentDAPS}
{proposes a noise annealing process to decouple consecutive samples in a sampling trajectory, which enables solvers to create errors made in earlier steps. We implement LatentDAPS from the official code, with modification of the solving process to use the linear flow-based backbone model StableDiffusion 3.0 \cite{esser2024scaling} and Euler solver.}

\paragraph{Hyper-parameter setting}
For all implementations, we use StableDiffusion 3.0 \cite{esser2024scaling} as our baseline model. Also, we set the shift factor of time scheduler to 4.0.

\begin{itemize}
    \item \textbf{PSLD} We set $\eta=1.0$ and $\gamma=0.1$ by following the original paper setting, and use 200 NFEs as in \cite{kim2023regularization}.
 { \item \textbf{Resample} We use the same resampling hyperparameter $\gamma \left(\frac{1-\bar\alpha_{t-1}}{\bar\alpha_t}\right)\left(1-\frac{\bar\alpha_t}{\bar\alpha_{t-1}}\right)$ with $\gamma=40$ as proposed in the original paper, reparameterizing $\bar\alpha_t$ as previously explained. Skip step size is set to 1 to perform sufficient hard data consistency steps.}
    \item \textbf{LatentDAPS} We use ODE solver steps $N_{ODE}=5$ and annealing scheduler $N_{A}=28$, resulting in a total NFE of 120. Total step number $N$ in Langevin Dynamics is set to 50, following the settings of the original paper \cite{zhang2024improving}.
    \item \textbf{FlowChef} For a constant step size for the FlowChef, we find the best configuration by grid search with 100 images. In consequence, we set the step size to 200 for the super-resolution tasks and 50 for the deblurring tasks. 
    \item \textbf{FlowDPS} For data consistency optimization, we use 3 steps of gradient descent with step size 15 for all tasks.
\end{itemize}

\section{Ablation Study}
\subsection{Analysis on $\beta_t$ and $\gamma_t$}
\label{app:beta}
{
In FlowDPS, we interpolate $\hat \vz_{0|t}$ and $\hat \vz_{0|t}(\vy)$ with coefficient $\gamma_t$ to ensure the data consistency update does not lead to excessive divergence for the flow model's trajectory. For the selection of $\gamma_t=\sigma_t$, we refer to the progression of the adaptive step size $\beta_t$ in \eqref{eqn:tildetweedie}.
Specifically, our likelihood gradient is applied to clean estimation as
\begin{align}
    \hat\vx_{0|t}(\vy) = \hat\vx_{0|t} - \beta_t \nabla_{\hat\vx_{0|t}} \log p(\vy|\hat\vx_{0|t})\notag
\end{align}
where $\beta_t=\frac{\zeta_t}{a_t}\frac{dt}{C_1(t)}$.
For a linear flow with $\sigma_t=t$,
\begin{align}
     \frac{\zeta_t}{a_t} = \dot a_t \frac{b_t}{a_t}\left(1-\frac{b_t}{a_t}\right)=\frac{-\sigma_t}{1-\sigma_t}\left(1-\frac{\sigma_t}{1-\sigma_t}\right)
\end{align}
and $\beta_t$ is expressed as
\begin{align}
    \beta_t = \frac{dt \sigma_t (2\sigma_t -1)}{(1-\sigma_t)^2(1-\sigma_{t+dt})}
\end{align}
with $dt=\sigma_{t+dt}-\sigma_t<0$.
Figure~\ref{fig:beta} illustrates the progress of $-\beta_t$ during the sampling process. As we mentioned in the main paper, it rapidly decreases to zero which leads to higher stepsize for likelihood gradient in the early stage.
In our method, we imitate the theory-driven behavior of step size for likelihood gradient by introducing interpolation coefficient $\gamma$ that emphasizes data consistency (i.e. likelihood gradient) in the early stage.
\begin{figure}[!h]
    \centering
    \includegraphics[width=.8\linewidth]{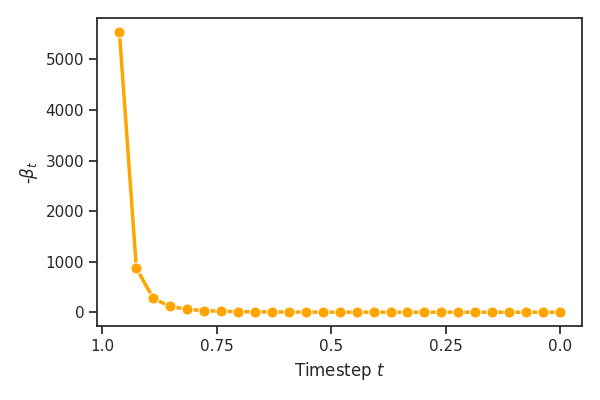}
    \vspace{-5mm}
    \caption{Evolution of $-\beta_t$ during sampling.}
    \label{fig:beta}
\end{figure}
}

\subsection{CFG scale and NFE}
\label{app:cfgnfe}
CFG scale $\lambda$ is an important factor for generating high-frequency details based on provided text prompts. When using low $\lambda$, the model fails to capture fine details, producing blurry results. Using high $\lambda$ enables the model to better generate fine details, but can also guide sample generation in a completely wrong direction when excessively high.
Figure \ref{fig:abl_cfg} in the Appendix shows how when using higher NFEs, the model requires more guidance for production of high-frequency details, causing higher values of $\lambda$ to perform better.

\begin{figure}[h]
    \centering
    \includegraphics[width=\linewidth]{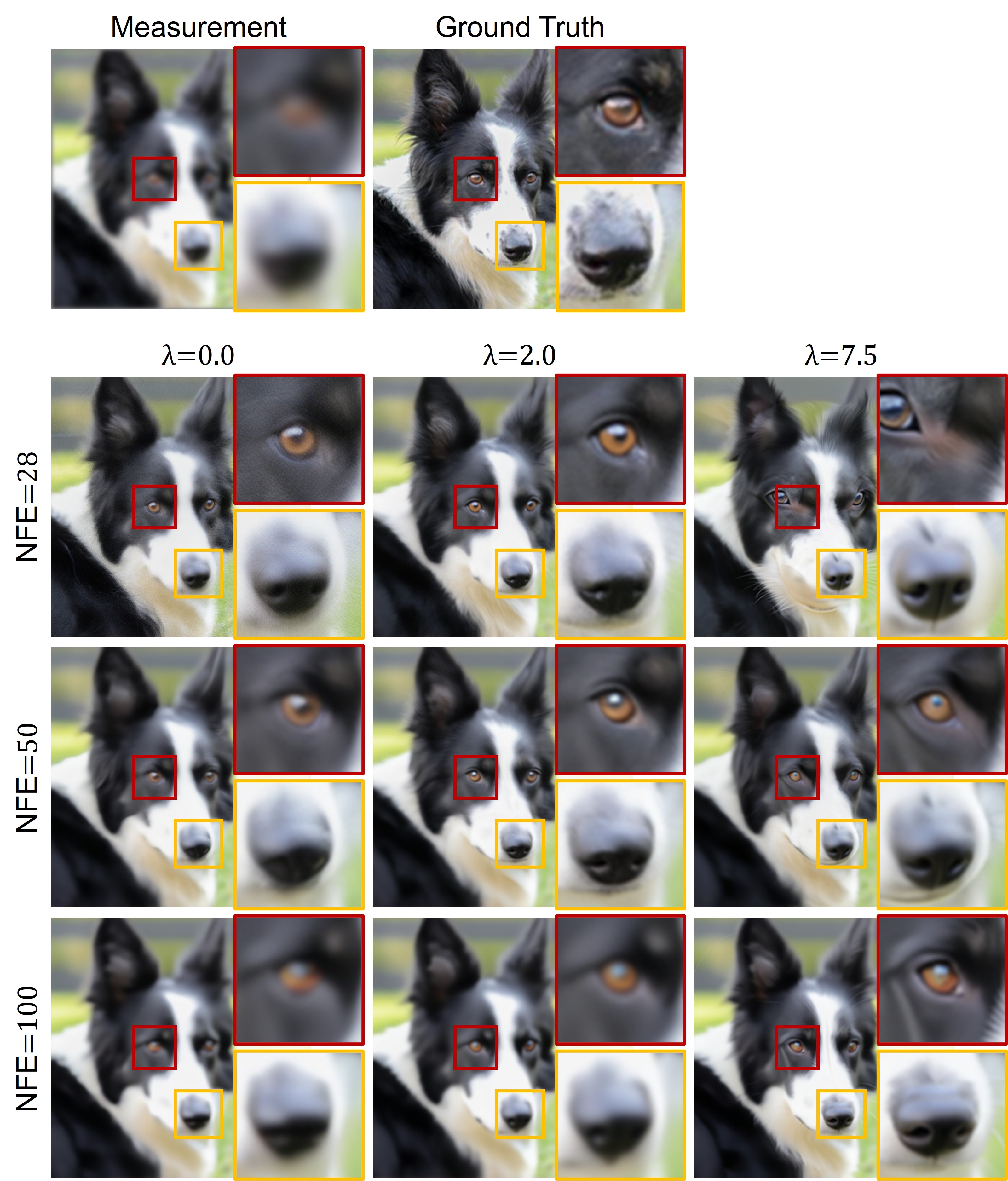}
    \caption{Relation between NFE and CFG scale $\lambda$. For varying NFEs (28, 50, 100), we study the impact of $\lambda$ (0.0, 2.0, 7.5).}
    \label{fig:abl_cfg}
    \vspace{-5mm}
\end{figure}

\section{Runtime Comparison}
Among baselines, PSLD takes the longest time as it requires computation of the Jacobian in terms of the transformer denoiser. In contrast, ReSample, FlowChef, and FlowDPS do not perform this computation due to approximation of the Jacobian, making them relatively efficient. Theoretically, the main computational bottleneck for these methods is backpropagation through the decoder and all other computations are relatively small, giving nearly identical runtime. Although LatentDAPS does not require backpropagation through a neural network, repeatedly solving the flow ODE and performing Langevin dynamics multiple times to achieve sufficient performance induces longer runtime compared to FlowDPS.

\section{Additional Results}
\label{app:additional}

\begin{figure*}[!t]
    \centering
    \includegraphics[width=0.95\linewidth]{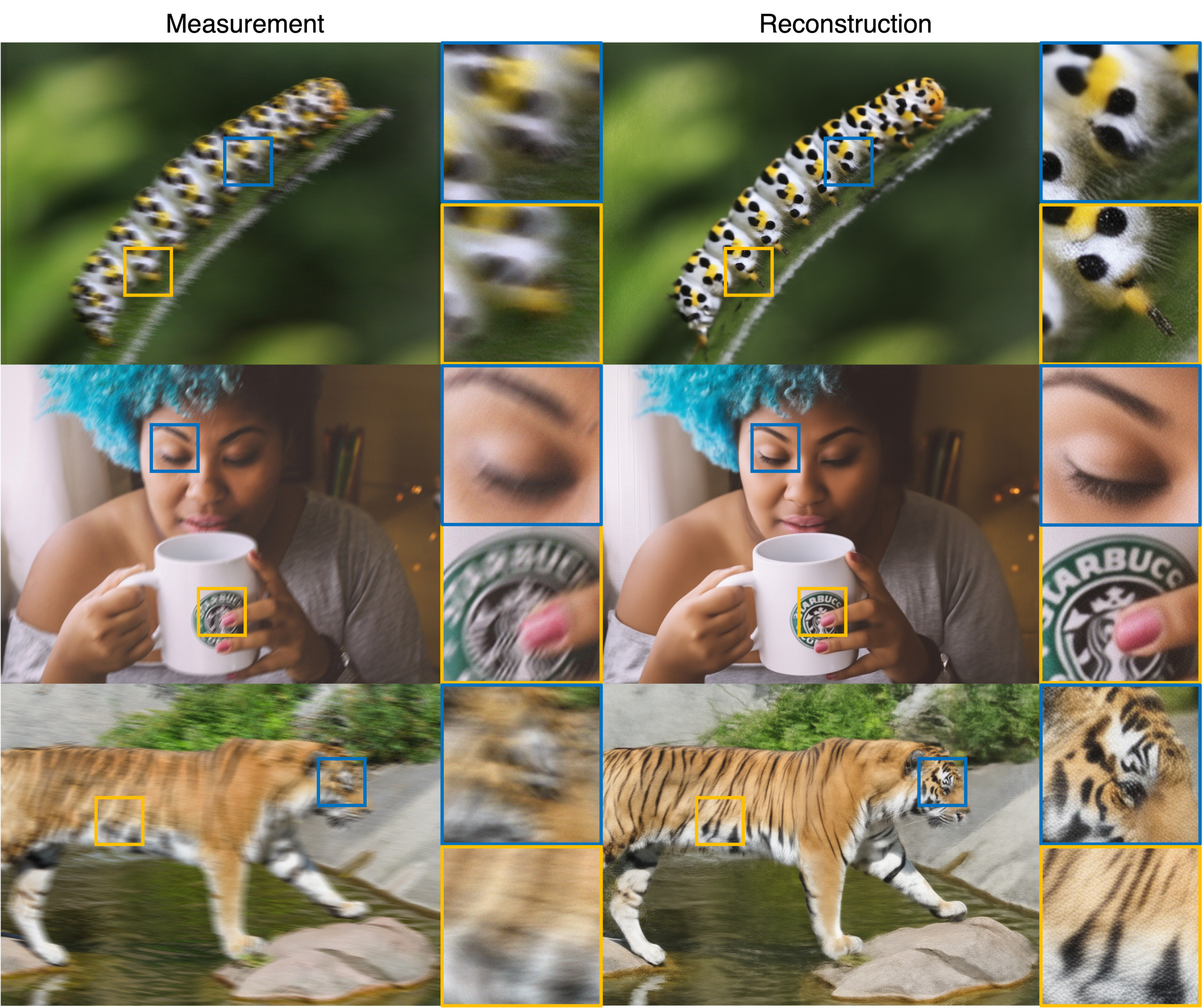}
    \caption{Motion deblur results by FlowDPS with SD3.0 for DIV2K images (1024x1408)}
    \vspace{-5mm}
    \label{fig:high-motion}
\end{figure*}

\subsection{Higher pixel resolution}

The diffusion process defined in latent space enables control over the generation process through various conditions, such as text, while also improving computational efficiency. Specifically, latent flow models like Stable Diffusion 3.0 can generate images with resolutions exceeding 1K pixels and support various aspect ratios beyond standard square formats.
We leverage this advantage of latent flow models to solve inverse problems, demonstrating the practical applicability of FlowDPS in real-world scenarios.
In Figure~\ref{fig:high-motion}, we address the motion blur problem using 1024×1408 images, obtained by cropping the central region of the DIV2K training set. For the text guidance, we use text description of the image extracted by LLaVA~\cite{NEURIPS2023_6dcf277e}. As demonstrated in the main experiment, FlowDPS successfully solves inverse problems defined in higher pixel resolutions without loss of generality.

\subsection{Inverse problem solving with FLUX}

The FlowDPS framework is designed for general affine conditional flows, enabling the construction of an inverse problem solver with various flow models.
In our main experiment, we demonstrate the efficiency of FlowDPS using a pre-trained linear conditional flow model, Stable Diffusion 3. To further illustrate the generality of FlowDPS, we also incorporate another open-source linear conditional flow model, FLUX.
Specifically, we set $\eta_t=0$ and use 10 iterations for data consistency optimization. The step size is set to 7.0 for super-resolution and 12.0 for deblurring, while all other settings remain consistent with Stable Diffusion 3.
Figure~\ref{fig:flux_sr} and \ref{fig:flux_blur} present the reconstruction results of FlowDPS with FLUX 1.0-schnell~\cite{flux2024} on DIV2K images. As in the main experiment, we utilize text prompts extracted from measurements using DAPE~\cite{wu2024seesr}. Regardless of the backbone model, FlowDPS effectively solves inverse problems.


\subsection{Additional Qualitative comparisons}
In this section, we further provide qualitative comparisons for the four inverse problems across the AFHQ and FFHQ datasets. Figures~\ref{fig:afhq_sravgpool}-\ref{fig:ffhq_mtblur} show that FlowDPS consistently achieves promising reconstruction results and outperforms all the baselines. 

\begin{figure*}[!t]
    \centering
    \includegraphics[width=\linewidth]{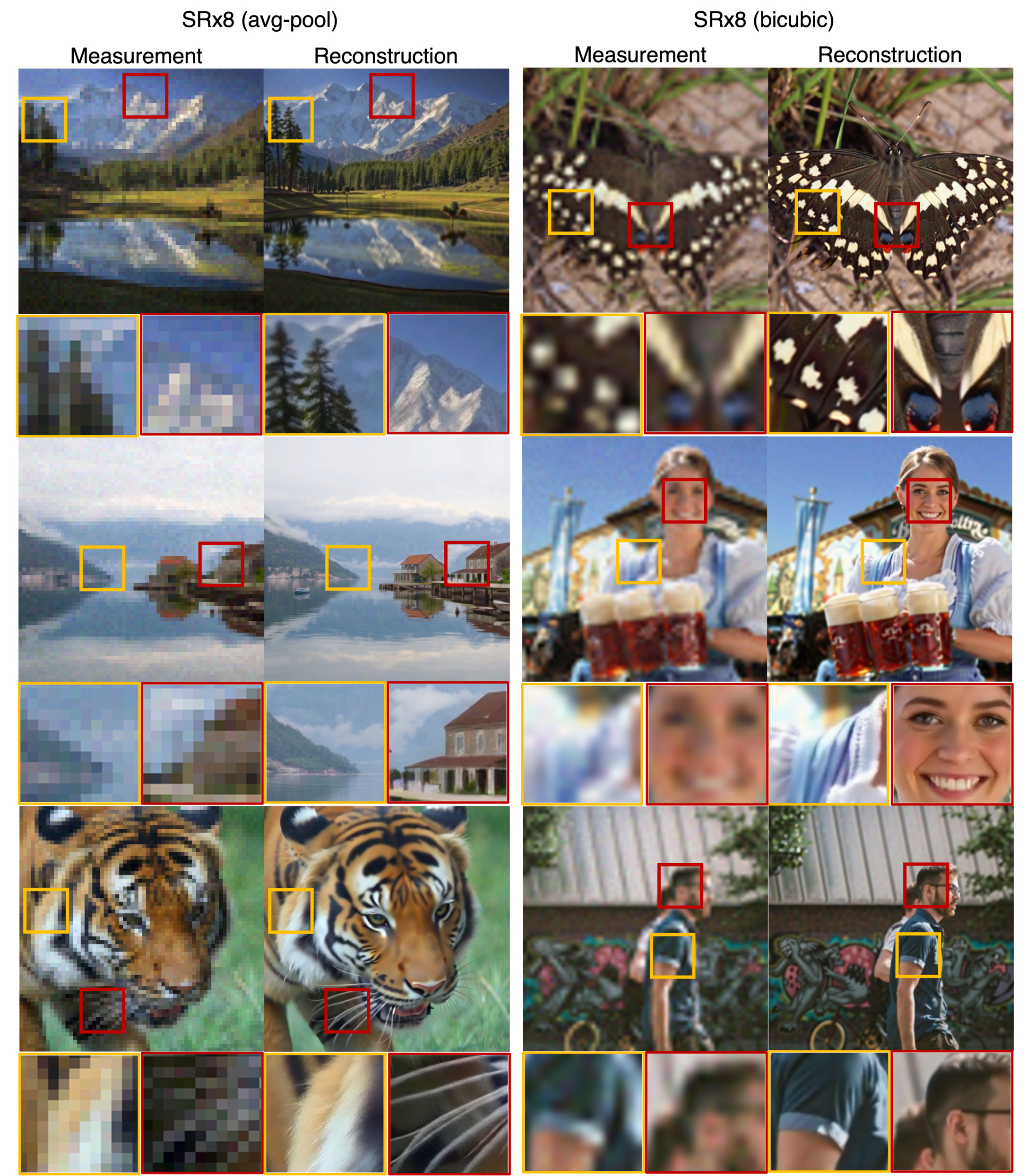}
    \caption{Super-resolution (x8) results by FlowDPS with Flux for DIV2K images (512x512)}
    \label{fig:flux_sr}
\end{figure*}
\begin{figure*}[!t]
    \centering
    \includegraphics[width=\linewidth]{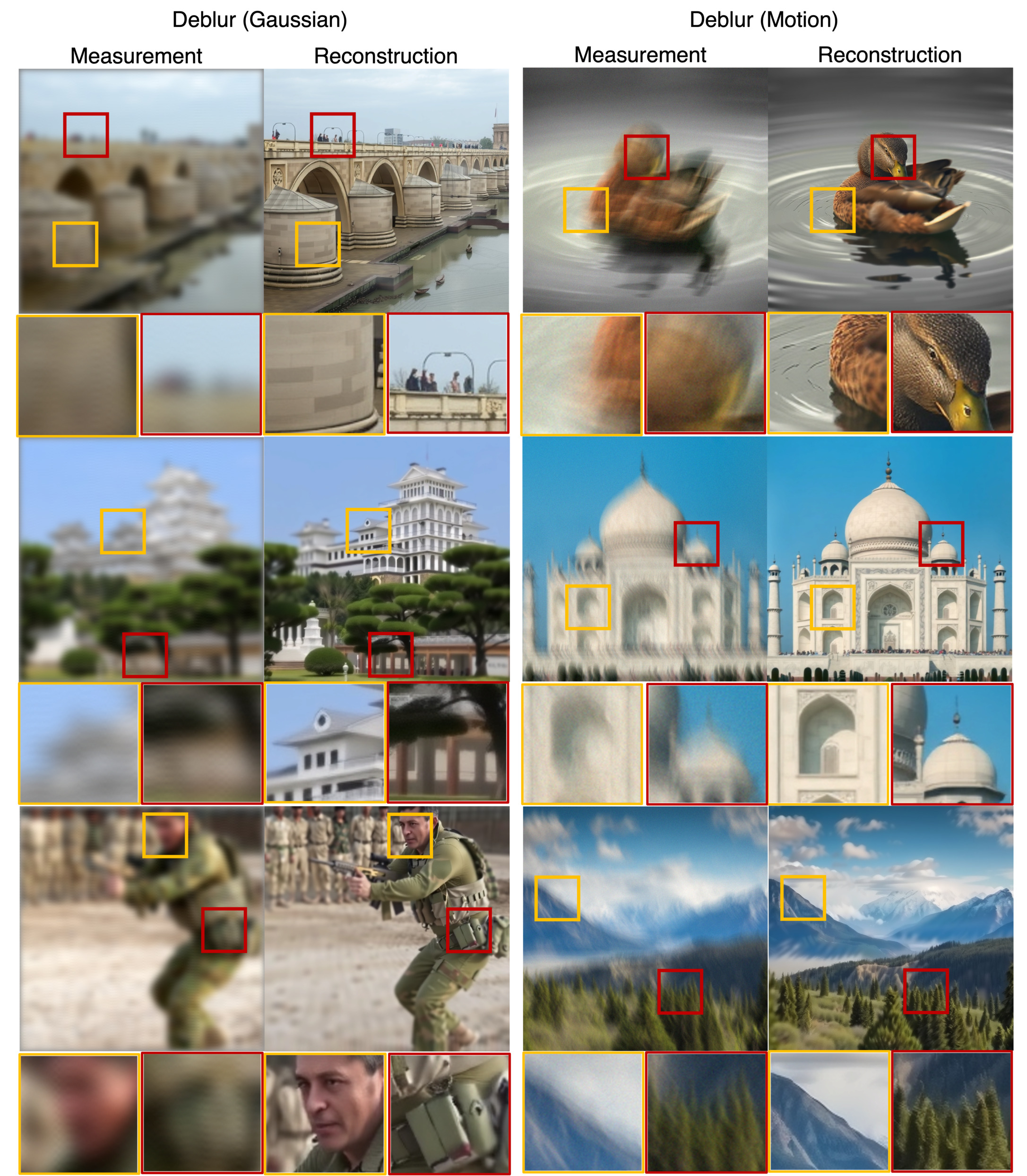}
    \caption{Deblurring results by FlowDPS with Flux for DIV2K images (512x512)}
    \label{fig:flux_blur}
\end{figure*}

\begin{figure*}[!t]
    \centering
    \includegraphics[width=.9\linewidth]{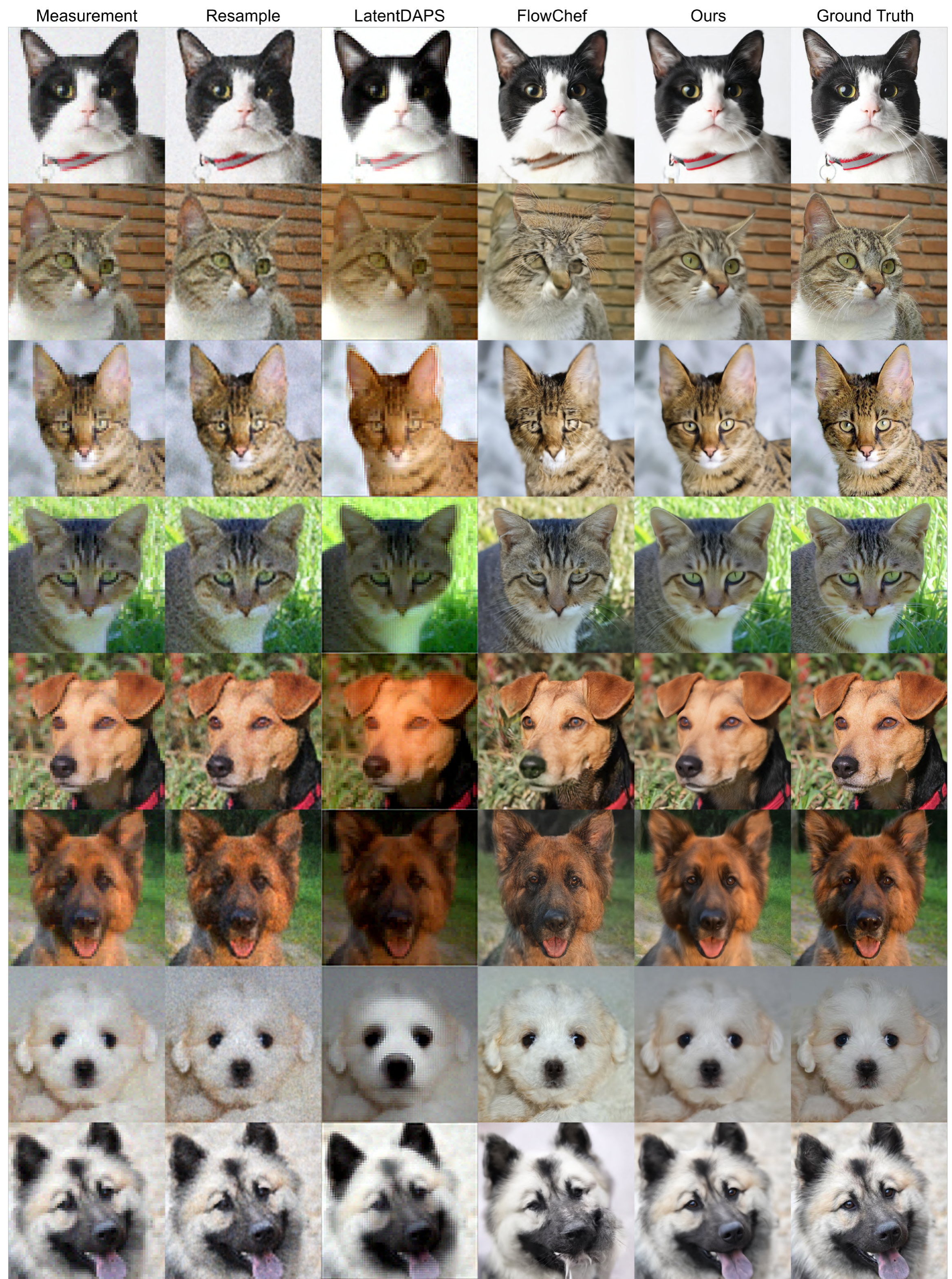}
    \caption{Qualitative comparison for x12 super-resolution from average pooling on the AFHQ dataset.}
    \label{fig:afhq_sravgpool}
\end{figure*}

\begin{figure*}[!t]
    \centering
    \includegraphics[width=.9\linewidth]{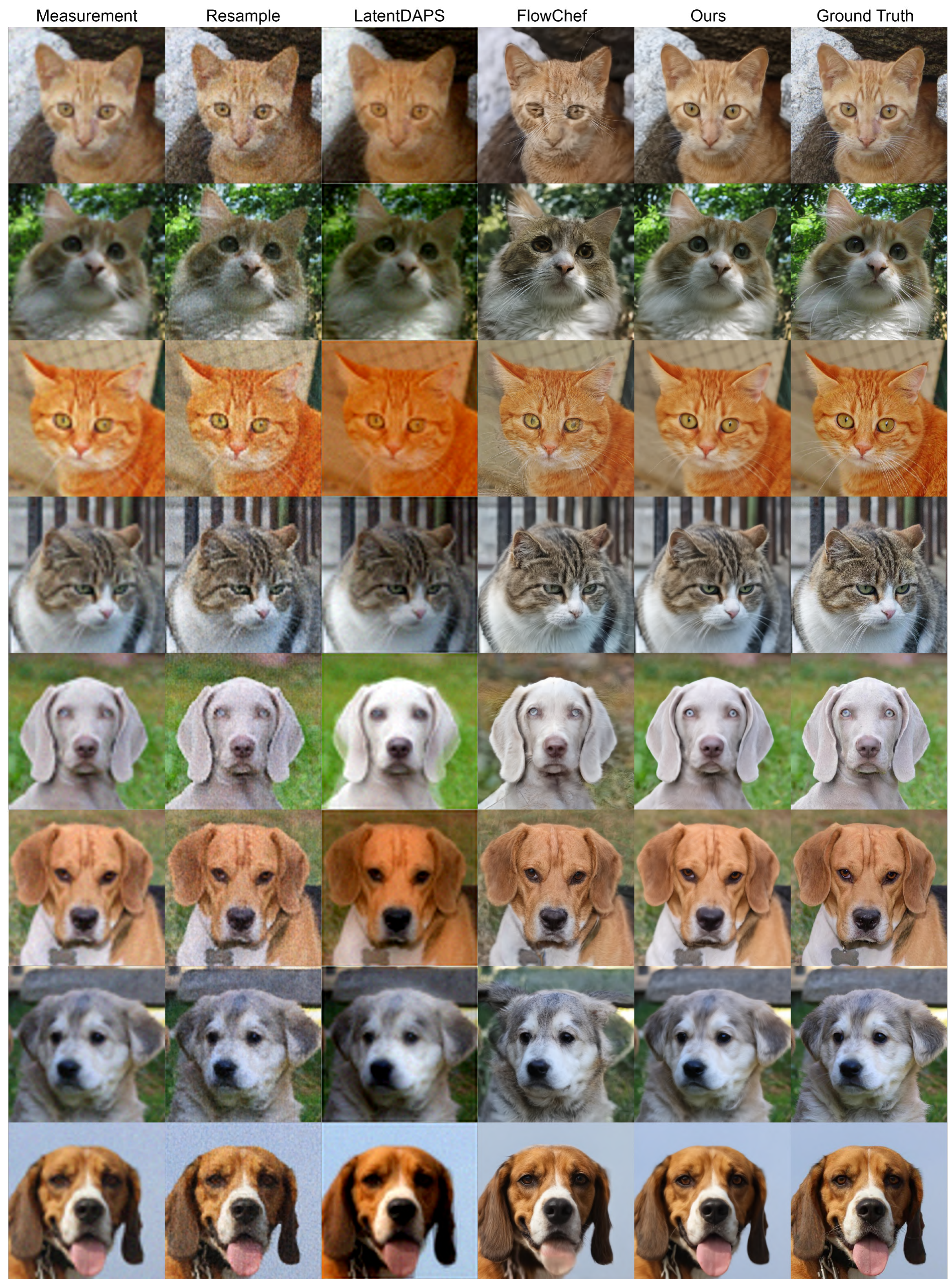}
    \caption{Qualitative comparison for x12 super-resolution from bicubic downsampling on the AFHQ dataset.}
    \label{fig:afhq_sravgbicubic}
\end{figure*}

\begin{figure*}[!t]
    \centering
    \includegraphics[width=.9\linewidth]{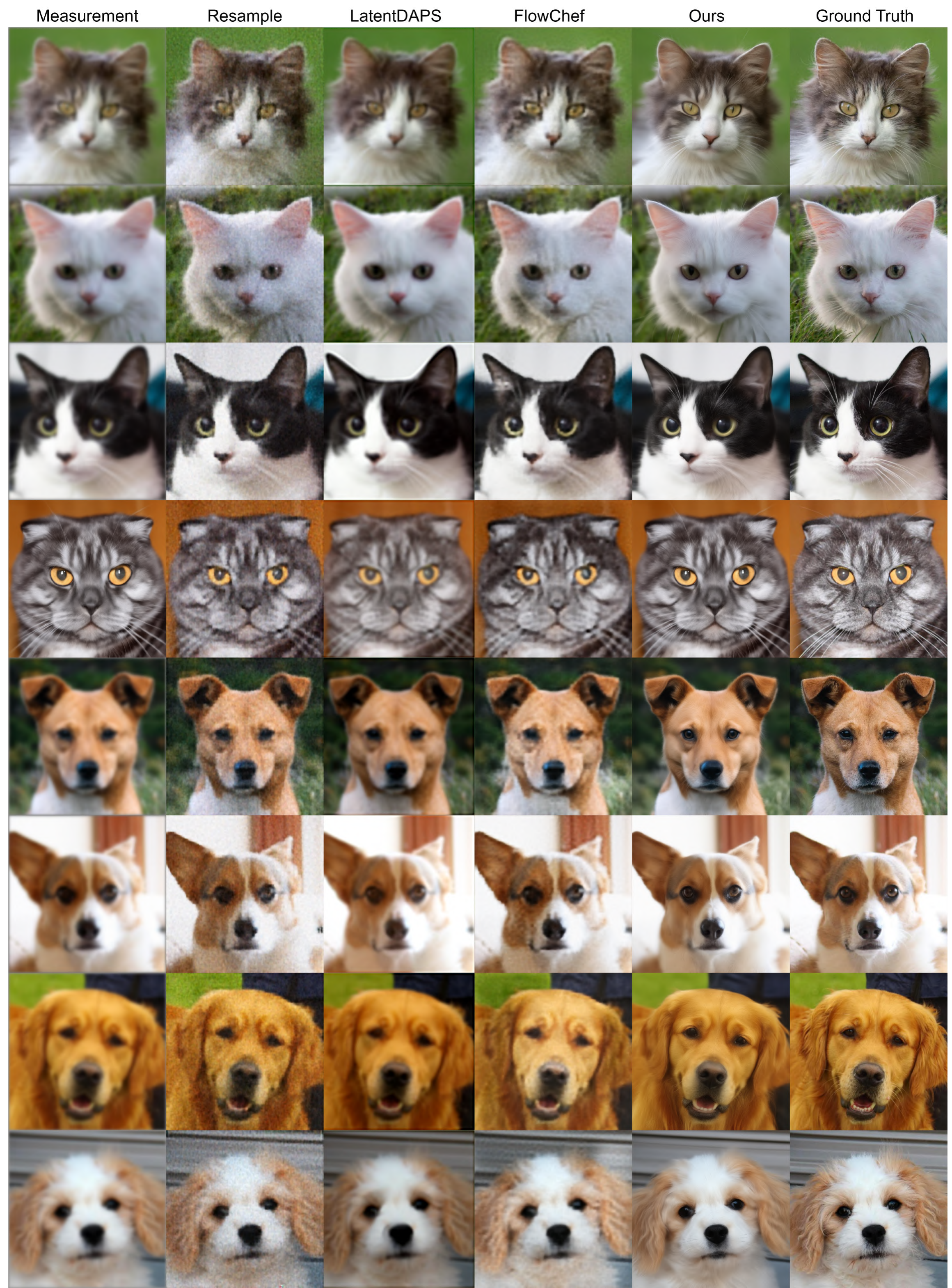}
    \caption{Qualitative comparison for Gaussian deblurring on the AFHQ dataset.}
    \label{fig:afhq_gsblur}
\end{figure*}

\begin{figure*}[!t]
    \centering
    \includegraphics[width=.9\linewidth]{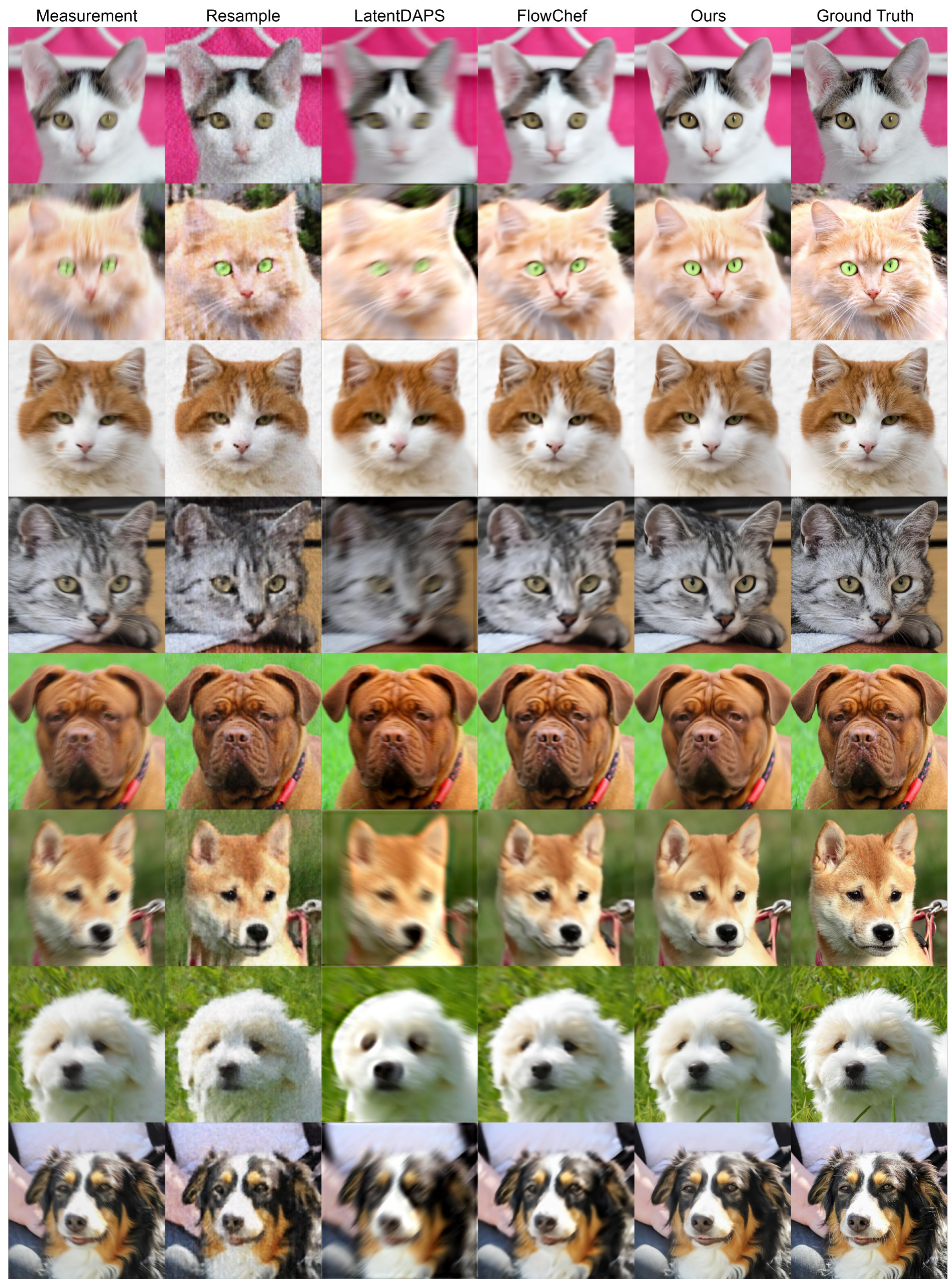}
    \caption{Qualitative comparison for motion deblurring on the AFHQ dataset.}
    \label{fig:afhq_mtblur}
\end{figure*}

\begin{figure*}[!t]
    \centering
    \includegraphics[width=.9\linewidth]{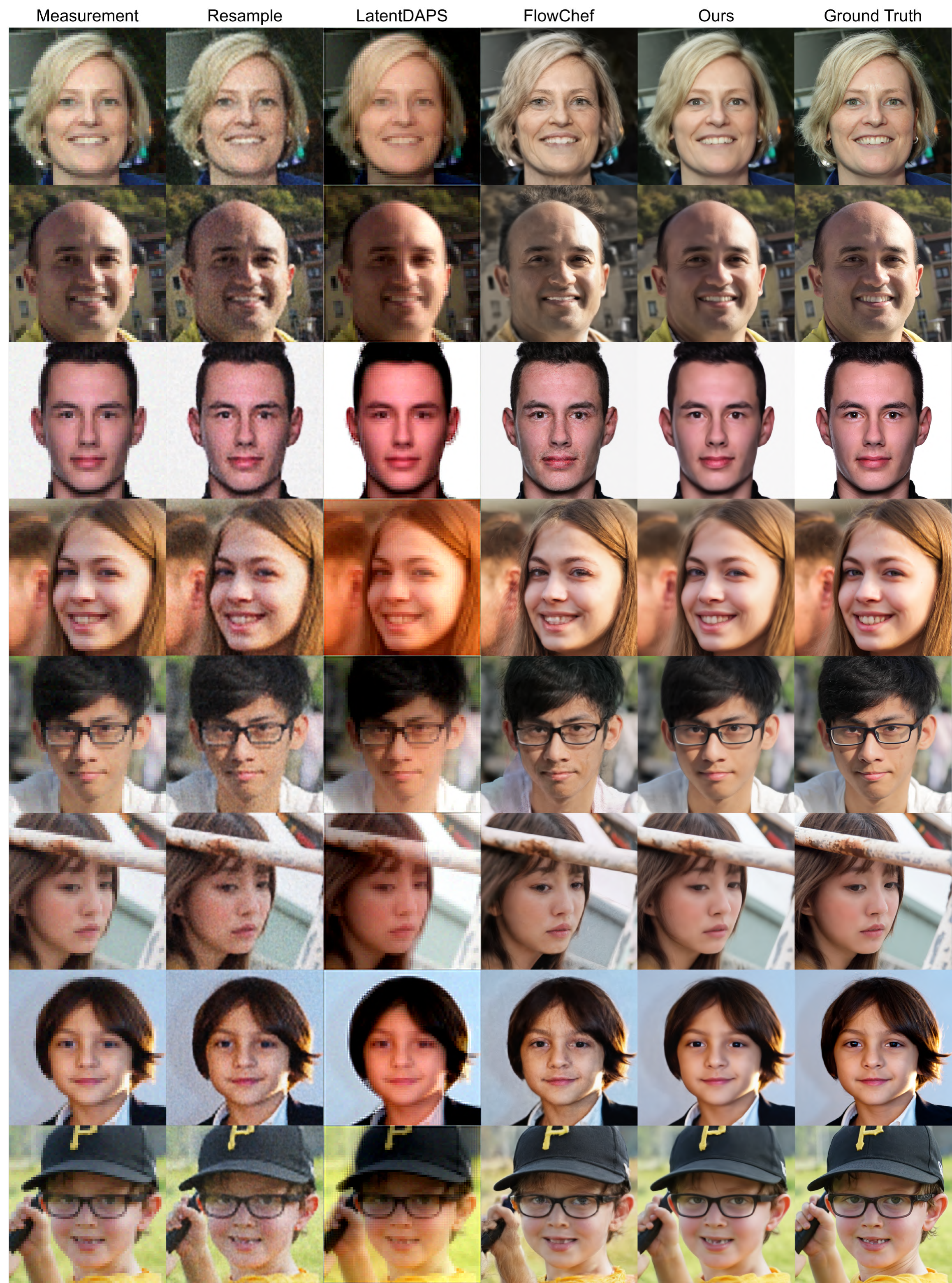}
    \caption{Qualitative comparison for x12 super-resolution from average pooling on the FFHQ dataset.}
    \label{fig:ffhq_sravgpool}
\end{figure*}

\begin{figure*}[!t]
    \centering
    \includegraphics[width=.9\linewidth]{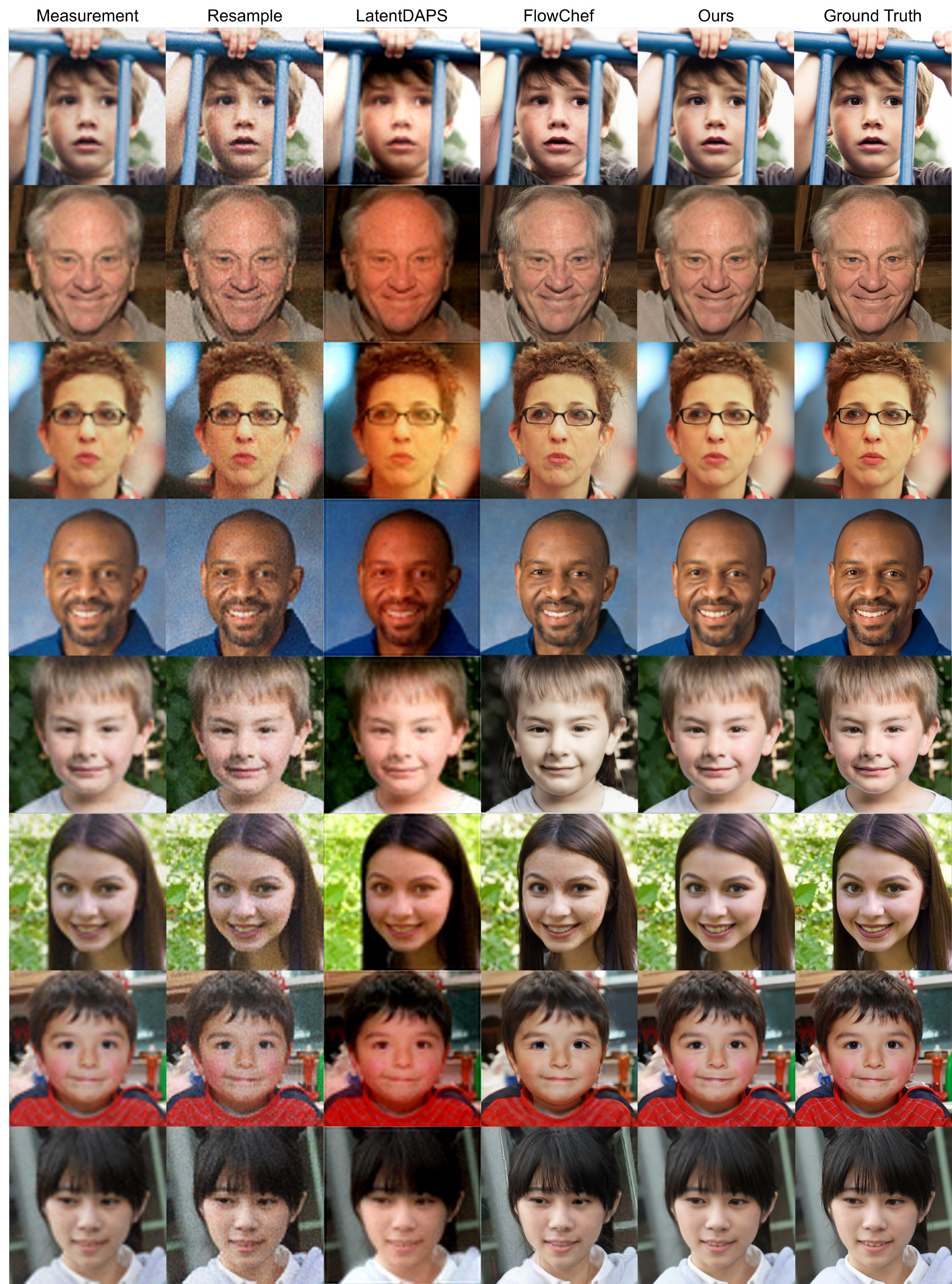}
    \caption{Qualitative comparison for x12 super-resolution from bicubic downsampling on the FFHQ dataset.}
    \label{fig:ffhq_sravgbicubic}
\end{figure*}

\begin{figure*}[!t]
    \centering
    \includegraphics[width=.9\linewidth]{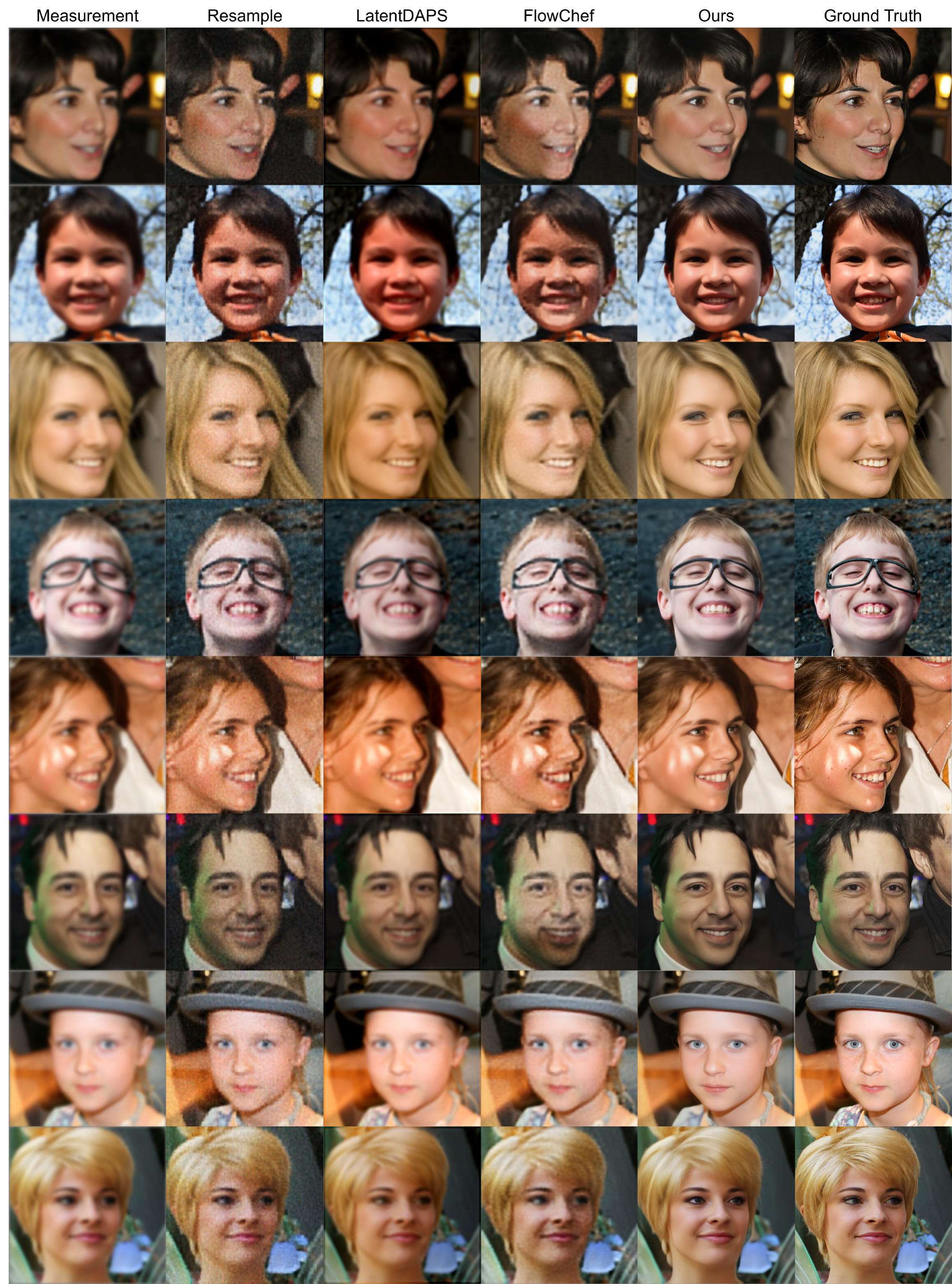}
    \caption{Qualitative comparison for Gaussian deblurring on the FFHQ dataset.}
    \label{fig:ffhq_gsblur}
\end{figure*}

\begin{figure*}[!t]
    \centering
    \includegraphics[width=.9\linewidth]{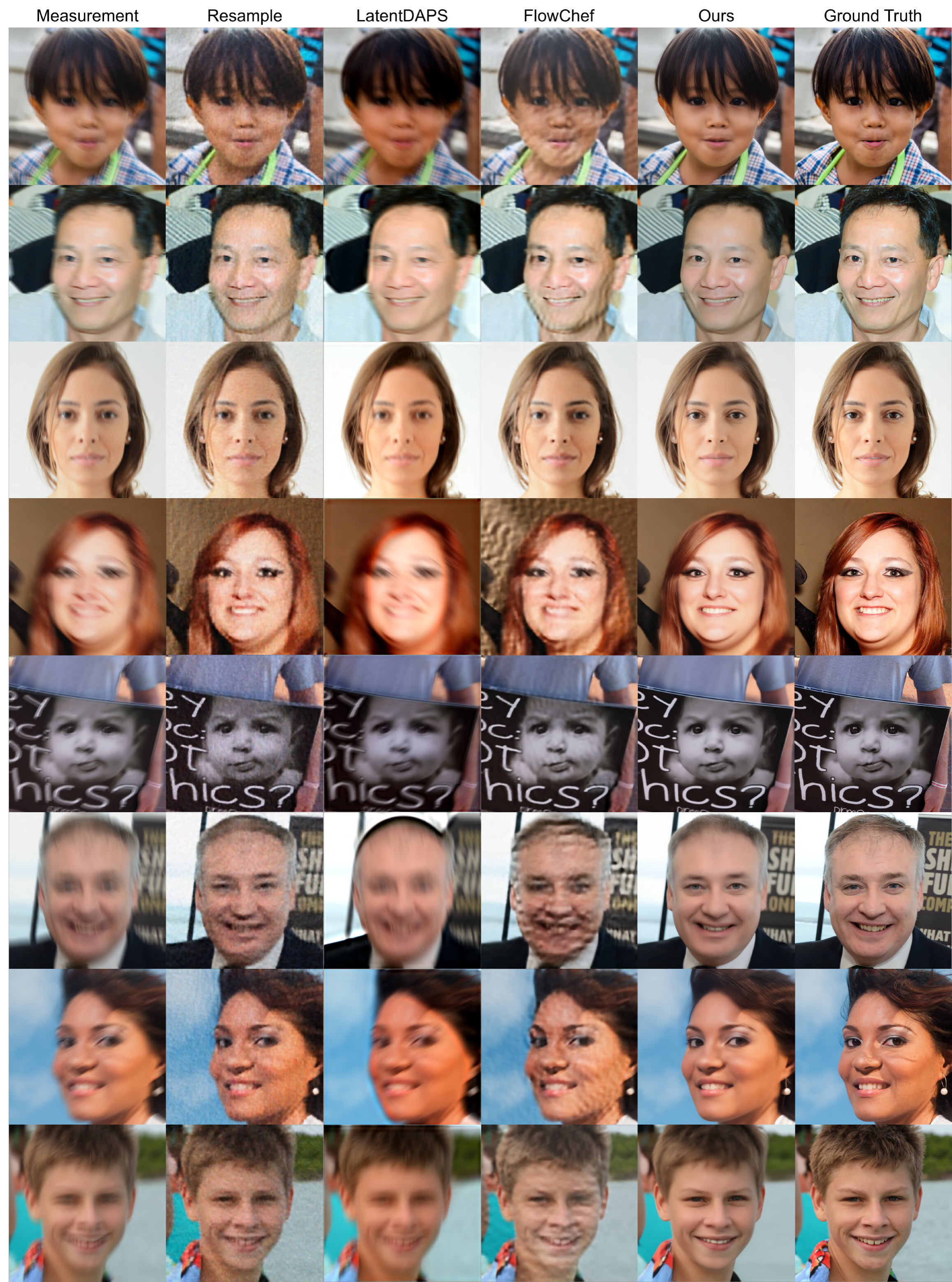}
    \caption{Qualitative comparison for motion deblurring on the FFHQ dataset.}
    \label{fig:ffhq_mtblur}
\end{figure*}


\end{document}

%% file: preamble.tex
\definecolor{Gray}{gray}{0.9}

%% file: math_commands.tex
\usepackage{amsmath,amsfonts,bm}

\def\eqref#1{equation~\ref{#1}}

\def\1{\bm{1}}

\def\rmA{{\mathbf{A}}}

\def\rmI{{\mathbf{I}}}

\def\vn{{\bm{n}}}

\def\vv{{\bm{v}}}

\def\vx{{\bm{x}}}
\def\vy{{\bm{y}}}
\def\vz{{\bm{z}}}

\DeclareMathAlphabet{\mathsfit}{\encodingdefault}{\sfdefault}{m}{sl}
\SetMathAlphabet{\mathsfit}{bold}{\encodingdefault}{\sfdefault}{bx}{n}

\DeclareMathOperator*{\argmin}{arg\,min}

\newcommand{\x}{{\boldsymbol x}}

\newcommand{\z}{{\boldsymbol z}}

\newcommand{\epsilonb}{{\boldsymbol \epsilon}}

\usepackage{capt-of}
\usepackage{diagbox}

\definecolor{C0}{rgb}{0.121569, 0.466667, 0.705882}
\definecolor{C1}{rgb}{1.000000, 0.498039, 0.054902}
\definecolor{C2}{rgb}{0.172549, 0.627451, 0.172549}
\definecolor{C3}{rgb}{0.839216, 0.152941, 0.156863}
\definecolor{C4}{rgb}{0.580392, 0.403922, 0.741176}
\definecolor{C5}{rgb}{0.549020, 0.337255, 0.294118}
\definecolor{C6}{rgb}{0.890196, 0.466667, 0.760784}
\definecolor{C7}{rgb}{0.498039, 0.498039, 0.498039}
\definecolor{C8}{rgb}{0.737255, 0.741176, 0.133333}
\definecolor{C9}{rgb}{0.090196, 0.745098, 0.811765}
\definecolor{trolleygrey}{rgb}{0.5, 0.5, 0.5}

\definecolor{BrickRed}{rgb}{0.6,0,0}
\definecolor{RoyalBlue}{rgb}{0,0,0.8}
\definecolor{Tdgreen}{rgb}{0,0.4,0.7}
\definecolor{pinegreen}{rgb}{0.0, 0.47, 0.44}
\definecolor{cornellred}{rgb}{0.7, 0.11, 0.11}
\definecolor{cadmiumgreen}{rgb}{0.0, 0.42, 0.24}
\definecolor{spirodiscoball}{rgb}{0.06, 0.75, 0.99}
\definecolor{mylightblue}{rgb}{0.85, 0.90, 0.94}
\definecolor{maroon}{cmyk}{0,0.87,0.68,0.32}

\usepackage{pifont}
\definecolor{cfg}{rgb}{0.906, 0.435, 0.318}
\definecolor{cfgpp}{rgb}{0.165, 0.616, 0.561}
\definecolor{cfgnull}{rgb}{0.208, 0.565, 0.953}

